\documentclass{article}

\usepackage{amsmath}
\usepackage{amsfonts}
\usepackage{amssymb}
\usepackage{amsthm}
\usepackage{xcolor}

\usepackage{microtype,amssymb,amsmath}
\usepackage{graphicx}
\usepackage{subfigure}
\usepackage{booktabs} 

\usepackage{hyperref}
\usepackage{algorithmic}

\usepackage[accepted]{icml2019}

\icmltitlerunning{Taylor Expansion Policy Optimization}

\let\originalleft\left
\let\originalright\right
\renewcommand{\left}{\mathopen{}\mathclose\bgroup\originalleft}
\renewcommand{\right}{\aftergroup\egroup\originalright}

\newcommand{\CommaBin}{\mathbin{\raisebox{0.5ex}{,}}}

\renewcommand{\epsilon}{\varepsilon}

\renewcommand{\bar}{\overline}

\newlength{\minipagewidth}
\newtheorem{predefinition}{Definition}

\newtheorem{theorem}{Theorem}

\newtheorem{preproposition}{Proposition}
\newenvironment{proposition}[1]
{
\begin{preproposition}
}{
\end{preproposition}
}

\usepackage{mathtools}

\renewcommand{\phi}{\varphi}
\renewcommand{\epsilon}{\varepsilon}

\newcommand{\transpose}{^\mathsf{\scriptscriptstyle T}}

\begin{document}

\twocolumn[
\icmltitle{Taylor Expansion Policy Optimization}



\icmlsetsymbol{equal}{*}

\begin{icmlauthorlist}
\icmlauthor{Yunhao Tang}{cu}
\icmlauthor{Michal Valko}{dm}
\icmlauthor{R\'emi Munos}{dm}
\end{icmlauthorlist}

\icmlaffiliation{cu}{Columbia University, New York, USA}
\icmlaffiliation{dm}{Google DeepMind, Paris, France}


\icmlkeywords{Machine Learning, ICML}

\vskip 0.3in
]



\printAffiliationsAndNotice{}  

\begin{abstract}
 In this work, we investigate the application of Taylor expansions in reinforcement learning. In particular, we propose \emph{Taylor expansion policy optimization}, a policy optimization formalism that generalizes prior work (e.g., TRPO) as a first-order special case. We also show that Taylor expansions intimately relate to off-policy evaluation. Finally, we show that this new formulation entails modifications which  improve the performance of several state-of-the-art distributed algorithms.
 
\end{abstract}


\section{Introduction}
\label{sec:intro}


Policy optimization is a major framework in model-free reinforcement learning (RL), with successful applications in challenging domains \citep{silver2016alphago,berner2019dota,vinyals2019alphastar}. Along with scaling up to powerful computational architectures \citep{mnih2016asynchronous,espeholt2018impala}, significant algorithmic performance gains are driven by insights into the drawbacks of na\" ive policy gradient algorithms \citep{sutton2000policy}. Among all algorithmic improvements, two of the most prominent are: \textit{trust-region policy search} \citep{schulman2015trust,schulman2017proximal,abdolmaleki2018maximum,song2019v} and \textit{off-policy corrections} \citep{munos2016safe,wang2016sample,gruslys2017reactor,espeholt2018impala}.

At the first glance, these two streams of ideas focus on orthogonal aspects of policy optimization. For trust-region policy search, the idea is to constrain the size of policy updates. This limits the deviations between consecutive policies and lower-bounds the performance of the new policy \citep{kakade2002approximately,schulman2015trust}. On the other hand, off-policy corrections require that we account for the discrepancy between target policy and behavior policy.
 \citet{espeholt2018impala} has observed that the corrections are especially useful for distributed algorithms, where behavior policy and target policy typically differ. Both algorithmic ideas have contributed significantly to stabilizing policy optimization.


 In this work, we partially unify both algorithmic ideas into a single framework. In particular, we noticed that as a ubiquitous approximation method, \textit{Taylor expansions}  share high-level similarities with both trust region policy search and off-policy corrections. To get high-level intuitions of such similarities, consider a simple 1D example of Taylor expansions. Given a sufficiently smooth real-valued function on the real line $f: \mathbb{R} \rightarrow \mathbb{R}$, the $k$-th order Taylor expansion of $f(x)$ at $x_0$ is  
    $f_k(x) \triangleq f(x_0) + \sum_{i=1}^k [f^{(i)}(x_0) / i!] (x-x_0)^i\!,$ 
where $f^{(i)}(x_0)$ are the $i$-th order derivatives at $x_0$. First, a common feature shared by Taylor expansions and trust-region policy search is the inherent notion of a trust region constraint. Indeed, in order for convergence to take place, a  \emph{trust-region constraint} is required $|x-x_0| < R(f,x_0)$\footnote{Here, $R(f,x_0)$ is the convergence radius of the expansions, which in general depends on the function $f$ and origin $x_0$.}.
Second, when using the truncation as an approximation to the original function $f_K(x)\approx f(x)$, Taylor expansions satisfy the requirement of off-policy evaluations: evaluate target policy with behavior data. Indeed, to evaluate the truncation $f_K(x)$ at any $x$ (\emph{target policy}), we only require the \emph{behavior policy} ``data'' at $x_0$ (i.e., derivatives $f^{(i)}(x_0)$).

Our paper proceeds as follows. In Section~\ref{sec:terl}, we start with a general result of applying Taylor expansions to Q-functions. When we apply the same technique to the RL objective, we reuse the general result and derive a higher-order policy optimization objective.   This leads to Section~\ref{sec:TayPO}, where we formally present the \emph{Taylor Expansion Policy Optimization} (TayPO)  and generalize prior work \citep{schulman2015trust,schulman2017proximal} as a first-order special case. In Section~\label{sec:uni}, we make clear connection between Taylor expansions and $Q(\lambda)$ \citep{harutyunyan_QLambda_2016}, a common return-based off-policy evaluation operator. Finally, in Section~\ref{sec:exp}, we show the performance gains due to the higher-order objectives across a range of state-of-the-art distributed deep RL agents.

\section{Taylor expansion for reinforcement learning}
\label{sec:terl}

Consider a Markov Decision Process (MDP) with state space $\mathcal{X}$ and action space $\mathcal{A}$. Let policy $\pi(\cdot|x)$ be a distribution over actions give state $x$. At a discrete time $t \geq 0$, the agent in state $x_t$ takes action $a_t \sim \pi(\cdot|x_t)$, receives reward $r_t \triangleq  r(x_t,a_t)$, and transitions to a next state $x_{t+1} \sim p(\cdot|x_t,a_t)$. We assume a discount factor $\gamma \in [0,1)$. Let $Q^\pi(x,a)$ be the action value function (Q-function) from state $x,$ taking action $a,$ and following policy $\pi$. For convenience, we use $ d_\gamma^\pi(\cdot,\cdot|x_0,a_0,\tau)$ to denote the discounted visitation distribution starting from state-action pair $(x_0,a_0)$ and following $\pi$, such that $d_\gamma^\pi(x,a|x_0,a_0,\tau) = (1-\gamma)\gamma^{-\tau} \sum_{t\geq \tau} \gamma^t P(x_t=x|x_0,a_0,\pi) \pi(a|x)$.
We thus have $Q^\pi(x,a) = (1-\gamma)^{-1}\mathbb{E}_{(x',a')\sim d_\gamma^\pi(\cdot,\cdot|x,a,0)}[r(x',a')]$. We focus on the RL objective of optimizing $\max_\pi J(\pi) \triangleq \mathbb{E}_{\pi,x_0}[\sum_{t\geq 0}\gamma^t r_t]$ starting from a fixed initial state $x_0$.

We define some useful matrix notation. For ease of analysis, we assume that $\mathcal{X}$ and $\mathcal{A}$ are both finite. Let $R \in \mathbb{R}^{|\mathcal{X}| \times |\mathcal{A}|}$ denote the reward function and $P^\pi\in \mathbb{R}^{|\mathcal{X}||\mathcal{A}| \times |\mathcal{X}||\mathcal{A}|} $ denote the transition matrix such that $P^\pi(x,a,y,b) \triangleq  p(y|x,a) \pi(b|y)$. We also define $Q^\pi\in \mathbb{R}^{|\mathcal{X}| \times |\mathcal{A}|}$ as the vector Q-function.
This matrix notation facilitates compact derivations, for example, the Bellman equation writes as $Q^\pi = R + \gamma P^\pi Q^\pi$.

\subsection{Taylor Expansion of Q-functions.} In this part, we state the Taylor expansion of Q-functions. Our motivation for the expansion is the following: Assume we aim to estimate $Q^\pi(x,a)$ for target policy $\pi$, and we only have access to data collected under a behavior policy $\mu$. Since $Q^\mu(x,a)$ can be readily estimated with the collected data, how do we approximate $Q^\pi(x,a)$ with $Q^\mu(x,a)$? 

Clearly, when $\pi = \mu$, then $Q^\pi = Q^\mu$. Whenever $\pi \neq \mu$, $Q^\pi$ starts to deviate from $Q^\mu$. Therefore, we apply Taylor expansion to describe the deviation $Q^\pi - Q^\mu$ in the orders of $P^\pi - P^\mu$. We provide the following result.
\begin{theorem}\label{thm:Taylor}

(proved in Appendix \ref{appendix:taylorexpansionq}) For any policies $\pi$ and~$\mu,$ and any $K \geq 1$,
we have
\begin{align}
    Q^\pi -Q^\mu=& \sum_{k=1}^K \left(\gamma (I-\gamma P^\mu)^{-1} (P^\pi - P^\mu)\right)^k Q^\mu \notag\\
    & +  \left(\gamma (I-\gamma P^\mu)^{-1} (P^\pi - P^\mu)\right)^{K+1} Q^\pi.\notag
\end{align}
In addition, if $||\pi-\mu||_1 \triangleq \max_{x} \sum_a |\pi(a|x) - \mu(a|x)| < (1-\gamma)/\gamma$, then the limit for $K\rightarrow\infty$ exists and we have
\begin{align} \label{eq:taylorexpansionoriginal}
Q^\pi - Q^\mu = \sum_{k=1}^\infty \underbrace{\left(\gamma (I - \gamma P^\mu)^{-1}(P^\pi - P^\mu)\right)^k Q^\mu}_{\triangleq U_k}.
\end{align}
\end{theorem}
The constraint between $\pi$ and $\mu$ is a result of the convergence radius of the Taylor expansion. The derivation follows by recursively applying the following equality:
$Q^\pi = Q^\mu + \gamma (I-\gamma P^\mu)^{-1}(P^\pi - P^\mu)Q^\pi.$ Please refer to the Appendix~\ref{appendix:taylorexpansionq} for a proof. For ease of notation, denote the $k$-th term on the RHS of Eq.\,\ref{eq:taylorexpansionoriginal} as $U_k$. This gives rise to $Q^\pi - Q^\mu = \sum_{k=1}^\infty U_k$. 

To represent $Q^\pi - Q^\mu$  explicitly with the deviation between $\pi$ and $\mu$, consider a diagonal matrix $D_{\pi/\mu}(x,a,y,b) \triangleq \pi(a|x)/\mu(b|y) \cdot \delta_{x=y,a=b}$ where $x,y \in \mathcal{X}, a,b\in \mathcal{A}$ and 
 where $\delta$ is the Dirac delta function; we restrict to the case where $\mu(a|x) > 0,\forall x,a$. This diagonal matrix $D_{\pi/\mu} - I$ is a measure of the deviation between $\pi$ and $\mu$. The above expression can be rewritten as 
\begin{equation}
Q^\pi - Q^\mu = \sum_{k=1}^\infty (\gamma (I - \gamma P^\mu)^{-1}P^\mu(D_{\pi/\mu} - I))^k Q^\mu.
\label{eq:taylorexpansion}
\end{equation}
We will see that the expansion in Eq.\,\ref{eq:taylorexpansion} is useful in Section\ref{sec:TayPO} when we derive  the Taylor expansion of the difference between the performances of two policies, $J(\pi) - J(\mu)$. In Section~\ref{sec:uni}, we also provide the connection between Taylor expansion and off-policy evaluation.

\subsection{Taylor expansion of reinforcement learning objective}
\label{ss:tayex}
When searching for a better policy, we are often interested in the difference $J(\pi) - J(\mu)$. With Eq.\,\ref{eq:taylorexpansion}, we can derive a similar Taylor expansion result for $J(\pi) - J(\mu)$. Let $\pi_t$ (resp., $\mu_t$) be the shorthand notation for $\pi(a_t|x_t)$ (resp., $\mu(a_t|x_t)$). Here, we formalize the orders of the expansion as the number of times that ratios $\pi_t / \mu_t - 1$ appear in the expression, e.g., the first-order expansion should only involve $\pi_t / \mu_t - 1$ up to the first order, without higher order terms, e.g., cross product $(\pi_t / \mu_t - 1)(\pi_{t^\prime} / \mu_{t^\prime} - 1)$. We denote the $k$-th order as $L_k(\pi,\mu)$ and by construction $J(\pi) - J(\mu) = \sum_{k=1}^\infty L_k(\pi,\mu)$. Next, we derive practically useful expressions for $L_k(\pi,\mu)$.



We provide a derivation sketch below and give the details in Appendix~\ref{appendix:taylorexpansionrl}. Let $\pi_0, \mu_0 \in \mathbb{R}^{|\mathcal{X}|\times|A|}$ be the joint distribution of policies and state at time $t=0$ such that $\pi_0(x,a) = \pi(a|x) \delta_{x=x_0}$. Note that the RL objective equivalently writes as $J(\pi) = V^\pi(x_0) = \sum_{a} \pi(a|x_0)Q^\pi(x_0,a)$ and can be expressed as an inner product $J(\pi) = \pi_0\transpose Q^\pi$. This allows us to import results from Eq.\,\ref{eq:taylorexpansion},
\begin{align}
    &J(\pi) - J(\mu) = \pi_0\transpose Q^\pi - \mu_0\transpose Q^\mu \label{eq:pgobjective}  \\ &= \left(\pi_0 - \mu_0\right)\transpose \left(Q^\mu + \sum_{k\geq1} U_k\right) + \mu_0\transpose \left(\sum_{k\geq1} U_k\right).  \nonumber
\end{align}

By reading off different orders of the expansion from the RHS of Eq.\,\ref{eq:pgobjective}, we derive
\begin{align}
    L_1(\pi,\mu) &= (\pi_0 - \mu_0)\transpose Q^\mu + \mu_0\transpose U_1,\label{eq:rlorders}\\
    L_k(\pi,\mu) &= (\pi_0 - \mu_0)\transpose U_{k-1} + \mu_0\transpose U_k, \ \forall k \geq 2. \nonumber
\end{align}
It is worth noting that the $k$-th order expansion of the RL objective $L_k(\pi,\mu)$ is a mixture of the $(k-1)$-th and $k$-th order Q-function expansions. This is because  $J(\pi)$ integrates $Q^\pi$ over the initial $\pi_0$ and the initial difference $\pi_0-\mu_0$ contributes one order of difference in $L_k(\pi,\mu)$.

Below, we illustrate the results for $k=1,2,$ and $k\geq 3$. To make the results more intuitive, we convert the matrix notation of Eq.\,\ref{eq:pgobjective} into explicit expectations under~$\mu$.



\paragraph{First-order expansion.} 
By converting $L_1(\pi,\mu)$ from Eq.\,\ref{eq:rlorders} into expectations, we get
\begin{align}
        \!\!\!\!\mathop{\mathbb{E}}_{ \substack{x,a\sim d_\gamma^\mu(\cdot,\cdot|x_0,a_0,0), \\a_0 \sim \mu(\cdot|x_0)}} \left[\left(\frac{\pi(a|x)}{\mu(a|x)}-1\right)Q^\mu(x,a)\right].\label{eq:firstorder}
\end{align}
To be precise $L_1(\pi,\mu) = (1-\gamma)^{-1}\times$ (Eq.\,\ref{eq:firstorder}) to account for the normalization of the distribution $d_\gamma^\mu$.  Note that $L_1(\pi,\mu)$
is exactly the same as surrogate objective proposed in prior work on scalable policy optimization
\citep{kakade2002approximately,schulman2015trust,schulman2017proximal}. Indeed, these works proposed to estimate and optimize such a surrogate objective at each iteration while enforcing a trust region. In the following, we generalize this objective with Taylor expansions.

\paragraph{Second-order expansion.} By converting $L_2(\pi,\mu)$ from Eq.\,\ref{eq:rlorders} into expectations, we get 
\begin{align}
     \mathop{\mathbb{E}}_{ \substack{x,a\sim d_\gamma^\mu(\cdot,\cdot|x_0,a_0,0), \\a_0 \sim \mu(\cdot|x_0)\\ x^\prime,a^\prime\sim d_\gamma^\mu(\cdot,\cdot|x,a,1) }} \left[\!\!\left(\frac{\pi(a|x)}{\mu(a|x)}\!\!-\!\!1\right)\left(\frac{\pi(a^\prime|x^\prime)}{\mu(a^\prime|x^\prime)}\!\!-\!\!1\right)Q^\mu\left(x^\prime,a^\prime\right)\right].\label{eq:secondorder}
\end{align}
Again, accounting for the normalization, $L_2(\pi,\mu) = \gamma(1-\gamma)^{-2}\times$(Eq.\,\ref{eq:secondorder}). To calculate the above expectation, we first start from $(x_0,a_0)$, and sample a pair $(x,a)$ from the discounted distribution $d_\gamma^\mu(\cdot,\cdot|x_0,a_0,0)$. Then, we use $(x,a)$ as the starting point and sample another pair from $d_\gamma^\mu(\cdot,\cdot|x,a,1)$. This implies that the second-order expansion can be estimated only via samples under $\mu$, which will be essential for policy optimization in practice.

It is worth noting that the second state-action pair  $(x^\prime,a^\prime) \sim d_\gamma^\mu(\cdot,\cdot|x,a,1)$ with the argument $\tau=1$ instead of $\tau=0$. This is because $L_k(\pi,\mu),k\geq 2$ only contains terms $\pi_t/\mu_t-1$ sampled across strictly different time steps.
\paragraph{Higher-order expansions.} Similarly to the first-order and second-order expansions, higher-order expansions are also possible by including proper higher-order terms in $\pi_t / \mu_t - 1$. For general $K\geq 1$, $L_K(\pi,\mu)$ can be expressed as (omitting the normalization constants)
\begin{align}
    \mathbb{E}_{(x^{(i)},a^{(i)})_{1\leq i\leq K}}\left[\prod_{i=1}^K \left(\frac{\pi(a^{(i)}|x^{(i)})}{\mu(a^{(i)} | x^{(i)})} -1 \right) Q^{\mu}(x^{(K)},a^{(K)})\right].
    \label{eq:orderk}
\end{align}
Here, $(x^{(i)},a^{(i)}),1\leq i\leq K$ are sampled sequentially, each following a discounted visitation distribution conditional on the previous state-action pair. We show their detailed derivations in Appendix \ref{appendix:taylorexpansionrl}.
Furthermore, we discuss the trade-off of different orders $K$ in Section~\ref{sec:TayPO}.


\paragraph{Interpretation \& intuition.} 
Evaluating $J(\pi)$ with data under $\mu$ requires importance sampling (IS) $J(\pi) = \mathbb{E}_{\mu,x_0}[(\Pi_{t\geq 0}\frac{\pi_t}{\mu_t})(\sum_{t\geq 0} \gamma^t r_t) ]$. In general, since $\pi$ can differ from $\mu$ at all $|\mathcal{X}||\mathcal{A}|$ state-action pairs, 
computing $J(\pi)$ exactly with full IS requires corrections at all steps along generated trajectories. First-order expansion (Eq.\,\ref{eq:firstorder}) corresponds to carrying out only one single correction at sampled state-action pair along the trajectories: Indeed, in computing Eq.\,\ref{eq:firstorder}, we sample a state-action pair $(x,a)$ along the trajectory and calculate one single IS correction $(\pi(a|x) / \mu(a|x) - 1)$. Similarly, the second-order expansion  (Eq.\,\ref{eq:secondorder}) goes one step further and considers the IS correction at two different steps $(x,a)$ and $(x^\prime,a^\prime)$. As such, Taylor expansions of the RL objective can be interpreted as increasingly tight approximations of the full IS correction. 

\section{Taylor expansion for policy optimization}
\label{sec:TayPO}
In high-dimensional policy optimization,  where exact algorithms such as dynamic programming are not feasible, it is necessary to learn from sampled data. In general, the sampled data are collected under a behavior policy $\mu$ different from the target policy $\pi$. For example, in trust-region policy search (e.g., TRPO, \citealp{schulman2015trust}; PPO, \citealp{schulman2017proximal}), $\pi$ is the new policy while $\mu$ is a previous policy; in asynchronous distributed algorithms \citep{mnih2016asynchronous,espeholt2018impala,horgan2018distributed,kapturowski2018recurrent}, $\pi$ is the learner policy while $\mu$ is delayed actor policy. In this section, we show the fundamental connection between trust-region policy search and Taylor expansions, and propose the general framework of \emph{Taylor expansion policy optimization } (TayPO).

 
\subsection{Generalized trust-region policy pptimization}
For policy optimization, it is necessary that the update function (e.g., policy gradients or surrogate objectives) can be estimated with sampled data under behavior policy $\mu$. Taylor expansions are a natural paradigm to satisfy this requirement. Indeed, to optimize $J(\pi)$, consider optimizing\footnote{Once again, the equality $J(\pi) = J(\mu) + \sum_{k=1}^\infty L_k(\pi,\mu)$ holds under certain conditions, detailed in Section 4.} 
\begin{align}
    \max_\pi\  J(\pi) = \max_\pi\  J(\mu) + \sum_{k=1}^\infty L_k(\pi,\mu).
    \label{eq:taylorobjective}
\end{align}
Though we have shown that for all $k,$ $L_k(\pi,\mu)$ are expectations under $\mu$, it is not feasible to unbiasedly estimate the  RHS of Eq.\,\ref{eq:taylorobjective} because it involves an infinite number of terms. In practice, we can truncate the objective up to $K$-th order $\sum_{k=1}^K L_k(\pi,\mu)$ and drop $J(\mu)$ because it does not involve~$\pi.$

However, for any fixed $K$, optimizing the truncated objective $\sum_{k=1}^K L_k(\pi,\mu)$ in an unconstrained way is risky: As $\pi,\mu$ become increasingly different, the approximation $J(\mu) + \sum_{k=1}^K L_k(\pi,\mu) \approx J(\pi)$ becomes more inaccurate and we stray away from optimizing $J(\pi),$ the objective of interest. The approximation error comes from the residual $E_K \triangleq \sum_{k=K+1}^\infty U_k$ --- to control the magnitude of the residual, it is natural  
to constrain $||\pi-\mu||_1 \leq \epsilon$ with some $\epsilon > 0$. Indeed, it is straightforward to show that \[||E_K||_\infty \leq \left(\frac{\gamma\epsilon}{1-\gamma} \right)^{K+1} \left(1-\frac{\gamma\epsilon}{1-\gamma}\right)^{-1} \frac{R_{\text{max}}}{1-\gamma}\CommaBin\] 
where $R_{\text{max}} \triangleq \max_{x,a} |r(x,a)|.$\footnote{Here we define $||E||_\infty \triangleq \max_{x,a} |E(x,a)|$.} Please see Appendix \ref{appendix:generalizedtrpo-residuals} for more detailed derivations. We formalize the entire local optimization problem as \emph{generalized trust-region policy optimization} (generalized TRPO),
\begin{equation}
   \max_\pi\  \sum_{k=1}^K L_k(\pi,\mu), \ \ 
   ||\pi-\mu||_1 \leq \epsilon. 
   \label{eq:generalizedtrpo}
\end{equation}
\paragraph{Monotonic improvement.} While maximizing the surrogate objective under trust-region constraints (Eq.\,\ref{eq:generalizedtrpo}), it is desirable to have performance guarantee on the true objective $J(\pi)$.
Below, Theorem \ref{thm:monotonic} gives such a result.
\begin{theorem}\label{thm:monotonic}(proved in Appendix \ref{sec:proofmonotonic}) When the policy $\pi$ is optimized based on the trust-region objective Eq.\,\ref{eq:generalizedtrpo} and $\epsilon < \frac{1-\gamma}{\gamma}$, the performance $J(\pi)$ is lower bounded as
\begin{align}
    J(\pi) &\geq J(\mu) + \sum_{k=1}^K L_k - G_K, \label{eq:monotonic} \\ \text{where}\  G_K &\triangleq \frac{1}{\gamma(1-\gamma)}\left(1-\frac{\gamma}{1-\gamma}\epsilon\right)^{-1}\left(\frac{\gamma\epsilon}{1-\gamma}\right)^{K+1} R_\text{max}.\notag
\end{align}
\end{theorem}
Note that if $\epsilon < (1-\gamma)/\gamma,$ then as $K\rightarrow \infty$, the gap $G_K \rightarrow 0$. Therefore, when optimizing $\pi$ based on Eq.\,\ref{eq:generalizedtrpo}, the performance $J(\pi)$ is always lower-bounded according to Eq.\,\ref{eq:monotonic}.

\paragraph{Connections to prior work on trust-region policy search.} The generalized TRPO extends the formulation of prior work, e.g., TRPO/PPO of \citet{schulman2015trust,schulman2017proximal}. Indeed, idealized forms of these algorithms are a special case for $K=1$, though for practical purposes the $\ell_1$ constraint is replaced by averaged KL constraints.\footnote{Instead of forming the constraints explicitly, PPO \citep{schulman2017proximal} enforces the constraints implicitly by clipping IS ratios.} 


\subsection{TayPO-$k$: Optimizing with $k$-th order expansion}
Though there is a theoretical motivation to use trust-region constraints for policy optimization \citep{schulman2015trust,abdolmaleki2018maximum}, such constraints are rarely explicitly enforced in practice in its most standard form (Eq.\,\ref{eq:generalizedtrpo}). Instead, trust regions are \textit{implicitly encouraged} via e.g., ratio clipping \citep{schulman2017proximal} or parameter averaging \citep{wang2016sample}. In large-scale distributed settings, algorithms already benefit from diverse sample collections for variance reduction of the parameter updates \citep{mnih2016asynchronous,espeholt2018impala}, which brings the desired stability for learning and makes trust-region constraints less necessary (either explicit or implicit). Therefore, we focus on the setting where no trust region is explicitly enforced. We introduce a new family of algorithm \textbf{TayPO-$k$}, which applies the $k$-th order Taylor expansions for policy optimization.


\paragraph{Unbiased estimations with variance reduction.} In practice, $L_k(\pi_\theta, \mu)$ as expectations under $\mu$ can be estimated as $\hat{L}_k(\pi_\theta,\mu)$ over a single trajectory. Take $K=2$ as an example: Given a trajectory $(x_t,a_t,r_t)_{t=0}^{\infty}$ by $\mu$, assume we have access to some estimates of $Q^\mu(x,a)$, e.g., cumulative returns. To generate a sample from $(x,a) \sim d_\gamma^\mu(x_0,a_0,0)$, we can first sample a random time from a geometric distribution with success probability $1-\gamma$, i.e., $t \sim \text{Geometric}(1-\gamma)$. Second, we sample another random time $t^\prime$ with geometric distribution $\text{Geometric}(1-\gamma)$ but conditional on $t^\prime \geq 1$.\footnote{As explained in Section~\ref{ss:tayex}, since $L_2(\pi,\mu)$ contains IS ratios at strictly different time steps, it is required that $t^\prime \geq 1$.}
Then, a single sample estimate of Eq.\,\ref{eq:secondorder} is given by
\begin{align*}
    \left(\frac{\pi(a_t|x_t)}{\mu(a_t|x_t)}-1\right)\left(\frac{\pi(a_{t+t^\prime}|x_{t+t^\prime})}{\mu(a_{t+t^\prime}|x_{t+t\prime})}-1\right) Q^\mu(x_{t+t^\prime},a_{t+t^\prime}).
\end{align*}
Further, the following shows the effect of replacing Q-values $Q^\mu(x,a)$ by advantages $A^\mu(x,a)\triangleq Q^\mu(x,a) - V^\mu(x)$.
\begin{theorem}\label{thm:variancereduction}

(proved in Appendix \ref{sec:proofvariancereduction}) The computation of $L_k(\pi,\mu)$ based on Eq.\,\ref{eq:orderk} is exact when replacing $Q^\mu(x,a)$ by $A^\mu(x,a)$, i.e. $L_k(\pi,\mu),k\geq 1$ can be expressed as
\begin{align}
\mathbb{E}_{(x^{(i)},a^{(i)})_{1\leq i\leq K}}\left[\prod_{i=1}^K \left(\frac{\pi(a^{(i)}|x^{(i)})}{\mu(a^{(i)} | x^{(i)})} -1 \right) A^{\mu}(x^{(K)},a^{(K)})\right].
\nonumber
\end{align}
\end{theorem}
In practice, when computing  $\hat{L}_k(\pi,\mu)$, replacing $\hat{Q}^\mu(x,a)$ by $\hat{A}^\mu(x,a)$ still produces an  unbiased estimate \textit{and} potentially reduces variance. This naturally recovers the result in prior work for $K=1$ \citep{schulman2015high}.


\paragraph{Higher-order objectives and trade-offs.} When $K\geq 3$, we can construct objectives with higher-order terms. The motivation is that with high $K$, $\sum_{k=1}^K L_k(\pi_\theta,\mu)$ forms a closer approximation to the objective of interest: $J(\pi) - J(\mu)$. Why not then have $K$ as large as possible? This comes at a trade-off. For example, let us compare $L_1(\pi_\theta,\mu)$ and $L_1(\pi_\theta,\mu) + L_2(\pi_\theta,\mu)$: Though $L_1(\pi_\theta,\mu) + L_2(\pi_\theta,\mu)$ forms a closer approximation to $J(\pi) - J(\mu)$ than $L_1(\pi)$ \emph{in expectation},
it could have higher variance during estimation when e.g., $L_1(\pi_\theta,\mu)$ and $L_2(\pi_\theta,\mu)$ have a non-negative correlation.
Indeed, as $K\rightarrow \infty$, $\sum_{k=1}^K L_k(\pi_\theta,\mu)$ approximates the full IS correction, which is known to have high variance \citep{munos2016safe}. 


\paragraph{How many orders to take in practice?}
Though the higher-order policy optimization formulation generalizes previous results \citep{schulman2015trust,schulman2017proximal} as an first-order special case, does it suffice to only include first-order terms in practice? 


\begin{figure}[t]
\includegraphics[width=8cm]{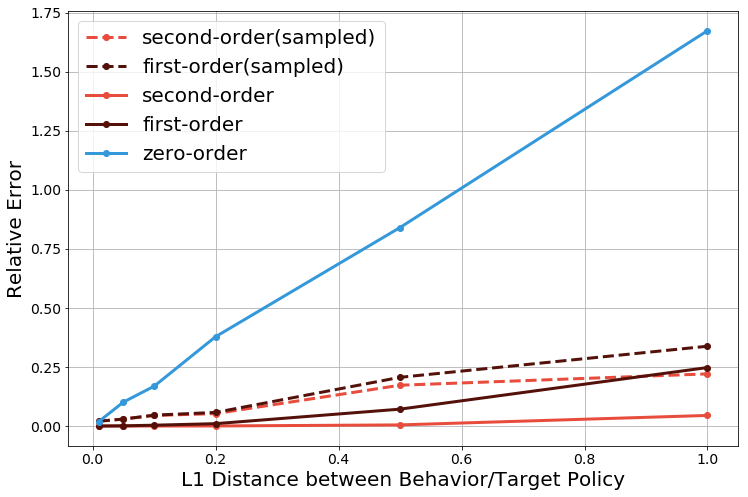}
\centering
\caption{Experiments on a small MDP. The $x$-axis measures $|\pi-\mu|_1$ and the $y$-axis shows the relative errors in off-policy estimates. All errors are computed analytically. Solid lines are computed with ground-truth rewards $R$ while dashed lines with estimates $\hat{R}$.}
\label{fig:smallmdp}
\end{figure}

To assess the effects of Taylor expansions, consider a policy evaluation problem on a random MDP (see Appendix \ref{appendix:experiment-randommdp} for the detailed setup): Given a target policy $\pi$ and a behavior policy $\mu$, the approximation error of the $K$-th order expansion is $e_K \triangleq Q^\pi - (Q^\mu + \sum_{k=1}^K U_k)$. In Figure \ref{fig:smallmdp}, We show the relative errors $||e_K||_1 / ||Q^\pi||_1$ 
as a function of $\epsilon=||\pi-\mu||_1$. Ground-truth quantities such as $Q^\pi$ are always computed analytically. Solid lines show results where all estimates are also computed analytically, e.g., $Q^\mu$ is computed as $(I-\gamma P^\mu)^{-1} R$. Observe that the errors decrease drastically as the expansion order $K\in\{0,1,2\}$ increases. 
To quantify how sample estimates impact the quality of approximations, we re-compute the estimates but with $R$ replaced by empirical estimates $\hat{R}$. Results are shown in dashed curves. Now comparing $K=1,2$, observe that both errors go up compared to their fully analytic counterparts - both become more similar when $\epsilon$ is small.


This provides motivations for second-order expansions. While first-orders are a default choice for common deep RL algorithms \citep{schulman2015trust,schulman2017proximal}, from the simple MDP example we see that the second-order expansions could potentially improve upon the first-order, \textit{ even with sample estimates}. 
\begin{algorithm}[h]
\begin{algorithmic}
\REQUIRE policy $\pi_\theta$ with parameter $\theta$ and $\alpha,\eta > 0$ \\
\WHILE{not converged}
\STATE 1. Collect partial trajectories $(x_t,a_t,r_t)_{t=1}^T$ under behavior policy $\mu$.
\STATE 2. Estimate on-policy advantage from the trajectories $\hat{A}^\mu(x_t,a_t)$.
\STATE 3. Construct first-order/second-order surrogate objective function $\hat{L}_1(\pi_\theta,\mu),  \hat{L}_2(\pi_\theta,\mu)$ according to Eq.\,\ref{eq:firstorder}, Eq.\,\ref{eq:secondorder} respectively, replacing $Q^\mu(x,a)$ by $\hat{A}^\mu(x,a)$.
\STATE 4. The full objective $\hat{L}_\theta \gets \hat{L}_1(\pi_\theta,\mu) + \hat{L}_2(\pi_\theta,\mu)$.
\STATE 5. Gradient update $\theta \gets \theta + \alpha \nabla_\theta \hat{L}_\theta$. 
\ENDWHILE
\caption{TayPO-$2$: Second-order policy optimization}
\end{algorithmic}
\end{algorithm}

\subsection{TayPO-$2$ --- Second-order policy optimization}
 From here onwards, we focus on TayPO-$2$. At any iteration, the data are collected under behavior policy $\mu$ in the form of partial trajectories $(x_t,a_t,r_t)_{t=1}^T$ of length $T$. The learner maintains a parametric policy $\pi_\theta$ to be optimized. First, we carry out advantage estimation $\hat{A}^\mu(x,a)$ for state-action pairs on the partial trajectories. This could be na\"ively estimated as $\hat{A}^\mu(x_t,a_t) = \sum_{t^\prime \geq t}^{T-1} r_{t^\prime}\gamma^{t^\prime-t} + V_\phi(x_T)\gamma^{T-t} - V_\phi(x_t)$ where $V_\phi(x)$ are value function baselines. One could also adopt more advanced estimation techniques such as \textit{generalized advantage estimation} (GAE, \citealp{schulman2015high}). Then, we construct surrogate objectives for optimization: the first-order component $\hat{L}_1(\pi_\theta,\mu)$ as well as second-order component $\hat{L}_2(\pi_\theta,\mu) - \hat{L}_1(\pi_\theta,\mu)$, based on Eq.\,\ref{eq:firstorder} and Eq.\,\ref{eq:secondorder} respectively. Note that we replace all $Q^\mu(x,a)$ by $\hat{A}^\mu(x,a)$ for variance reduction. 

Therefore, our final objective function becomes  
\begin{align}
\hat{L}_\theta \triangleq \hat{L}_1(\pi_\theta,\mu) + \hat{L}_2(\pi_\theta,\mu). \label{eq:TayPO2}
\end{align}
 The parameter is updated via gradient ascent $\theta \leftarrow \theta + \alpha \nabla \hat{L}_\theta$. 
 Similar ideas can be applied to \textit{value-based algorithms}, for which we provide details in Appendix \ref{appendix:secondorderr2d2}.

\section{Unifying the concepts: Taylor expansion as return-based off-policy evaluation}
\label{sec:uni}
So far we have made the connection between Taylor expansions and TRPO. On the other hand, as introduced in Section~\ref{sec:intro}, Taylor expansions can also be intimately related to \textit{off-policy evaluation}. Below, we formalize their connections. With Taylor expansions, we provide a consistent and unified view of TRPO and off-policy evaluation.


\subsection{Taylor expansion as off-policy evaluation}
In the general setting of off-policy evaluation, the data is collected under a behavior policy $\mu$ while the objective is to evaluate $Q^\pi$. \textit{Return-based off-policy evaluation operators} \citep{munos2016safe} are a family of operators $\mathcal{R}_c^{\pi,\mu}: \mathbb{R}^{|\mathcal{X}||\mathcal{A}|} \mapsto \mathbb{R}^{|\mathcal{X}||\mathcal{A}|}$, indexed by (per state-action) trace-cutting coefficients $c(x,a)$, a behavior policy $\mu$ and a target policy $\pi,$
$$
\mathcal{R}_c^{\pi,\mu}Q\triangleq Q+(I-\gamma P^{c\mu})^{-1}(r+\gamma P^{\pi}Q-Q), 
$$
where $P^{c\mu}$ is the \textit{(sub)-probability transition kernel} for policy $c(x',a')\mu(a'|x')$. 
Starting from any Q-function $Q$, repeated applications of the operator will result in convergence to $Q^\pi$, i.e.,
\begin{equation*}
    (\mathcal{R}_c^{\pi,\mu})^K Q \rightarrow Q^\pi,
\end{equation*}
as $K \rightarrow \infty$, subject to certain conditions on $c(x,a)$. To state the main results, recall that  Eq.\,\ref{eq:taylorexpansion} rewrites as 
$
    Q^\pi =  \mathop{\lim}_{ K\rightarrow \infty} \big( Q^\mu + \sum_{k=1}^K U_k \big).
$
In practice, we take a finite $K$ and use the approximation $Q^\mu + \sum_{k=1}^K U_k \approx Q^\pi$.

Next, we state the following result establishing a connection between $K$-th order Taylor expansion and the return-based off-policy operator applied $K$ times.
\begin{theorem}(proved in Appendix~\ref{sec:proofreturn})
\label{thm:off-policy}
For any $K\geq 1$, any policies $\pi$ and $\mu$,
\begin{align}
    Q^\mu + \sum_{k=1}^K U_k = (\mathcal{R}_1^{\pi,\mu})^K Q^\mu, \label{eq:equivalence}
\end{align}
where $\mathcal{R}_1^{\pi,\mu}$ is short for $c(x,a) \equiv 1$.
\end{theorem}
Theorem~\ref{thm:off-policy} shows that when we approximate $Q^\pi$ by the Taylor expansion up to the $K$-th order, $Q^\mu + \sum_{k=1}^K U_k$, it is equivalent to generating an approximation by  $K$ times applying the off-policy evaluation operator $\mathcal{R}_1^{\pi,\mu}$ on $Q^\mu$. We also note that the  off-policy evaluation operator in Theorem~\ref{thm:off-policy} is the $Q(\lambda)$ operator \cite{harutyunyan_QLambda_2016} with $\lambda=1$.\footnote{As a side note, we also show that the advatnage estimation method GAE \citep{schulman2015high} is highly related to the $Q(\lambda)$ operator in Appendix \ref{appendix:gae}.}

\paragraph{Alternative proof for $Q(\lambda)$ convergence for $\lambda=1$.} Since Taylor expansions converge within a convergence radius, which in this case corresponds to $||\pi-\mu||_1  <(1-\gamma)/\gamma$, it implies that $Q(\lambda)$ with $\lambda=1$ converges when this condition holds. In fact, this coincides with the condition deduced by \citet{harutyunyan_QLambda_2016}.\footnote{Note that this alternative proof only works for the case where the initial $Q_{\text{init}} = Q^\mu$.}

\subsection{An operator view of trust-region policy optimization} With the connection between Taylor expansion and off-policy evaluation, along with the connection between Taylor expansion and TRPO (Section~\ref{sec:TayPO}) we give a novel interpretation of TRPO:  The $K$-th order generalized TRPO is approximately equivalent to iterating $K$ times the off-policy evaluation operator $\mathcal{R}_1^{\pi,\mu}$.

To make our claim explicit, recall the RL objective in matrix form is $J(\pi) = \pi_0\transpose Q^\pi$. 
Now consider approximating $Q^\pi$ by applying the evaluation operator $\mathcal{R}_{1}^{\pi,\mu}$ to $Q^\mu$, iterating $K$ times. This produces the surrogate objective $\pi_0\transpose (\mathcal{R}_{1}^{\pi,\mu})^KQ^\mu
\approx J(\mu) + \sum_{k=1}^K L_k(\pi,\mu)$, approximately equivalent to that of the generalized TRPO (Eq.\,\ref{eq:generalizedtrpo}).\footnote{The $k$-th order Taylor expansion of $Q^\pi$ is slightly different from that of the RL objective $J(\pi)$ by construction; see Appendix~\ref{appendix:taylorexpansionq} for details.} As a result, the generalized TRPO (including TRPO; \citealp{schulman2015trust}) can be interpreted as approximating the exact RL objective $J(\pi$), by $K$ times iterating  the evaluation operator $\mathcal{R}_{1}^{\pi,\mu}$ on $Q^\mu$ to approximate $Q^\pi$. When does this evaluation operator converge? Recall that $\mathcal{R}_{1}^{\pi,\mu}$ converges when $||\pi - \mu||_1 < (1-\gamma)/\gamma$, i.e., there is a trust region constraint on $\pi,\mu$. This is consistent with the motivation of generalized TRPO discussed in Section~\ref{sec:TayPO}, where a trust region is required for monotonic improvements.

\section{Experiments}
\label{sec:exp}
We evaluate the potential benefits of applying second-order expansions in a diverse set of scenarios. In particular, we test if the second-order correction helps with \textbf{(1)} \textit{policy-based} and \textbf{(2)} \textit{value-based} algorithms.

 In large-scale experiments, to take advantage of computational architectures, actors ($\mu$) and learners ($\pi$) are not perfectly synchronized. For case~\textbf{(1)}, in Section \ref{onpolicy}, we show that even in cases where they almost synchronize ($\pi \approx \mu$), higher-order corrections are still helpful. Then, in Section \ref{offpolicy}, we study how the performance of a general distributed policy-based agent (e.g., IMPALA, \citealp{espeholt2018impala}) is influenced by the discrepancy between actors and learners. For case \textbf{(2)}, in Section~\ref{value},  we show the benefits of second-order expansions in with a state-of-the-art value-based agent R2D2 \citep{kapturowski2018recurrent}.

\paragraph{Evaluation.} All evaluation environments are done on the entire suite of Atari games \citep{bellemare2013arcade}. We report human-normalized scores for each level, calculated as $z_i = (r_i - o_i) / (h_i - o_i)$, where $h_i$ and $o_i$ are the performances of human and a random policy on level $i$ respectively; with details in Appendix~\ref{appendix:experiment-evaluation}.


\paragraph{Architecture for distributed agents.} Distributed agents generally consist of a central learner and multiple actors \citep{nair2015massively,mnih2016asynchronous,babaeizadeh2016reinforcement,barth2018distributed,horgan2018distributed}. We focus on two main setups: \textbf{Type I} includes agents such as IMPALA \citep{espeholt2018impala} (see blue arrows in Figure~\ref{fig:architecture} in Appendix~\ref{appendix:experiment-architecture}). See Section~\ref{onpolicy} and Section~\ref{ss:gendist}; \textbf{Type II} includes agents such as R2D2 (\citealp{kapturowski2018recurrent}; see orange arrows in Figure~\ref{fig:architecture} in Appendix~\ref{appendix:experiment-architecture}). See Section~\ref{ss:disvb}. 
We provide details on hyper-parameters of experiment setups in respective subsections in Appendix \ref{appendix:experiment}.

\paragraph{Practical considerations.}
We can extend the TayPO-$2$ objective (Eq.\,\ref{eq:TayPO2}) to $\hat{L}_\theta = \hat{L}_1(\pi_\theta,\mu) + \eta\hat{L}_2(\pi_\theta,\mu)$ with $\eta > 0$. By choosing $\eta$, one achieves bias-variance trade-offs of the final objective and hence the update. We found $\eta=1$ (exact TayPO-$2$) working reasonably well. See Appendix \ref{appendix:experiment-taypo2} for the ablation study on~$\eta$ and further details.



\subsection{Near on-policy policy optimization}
\label{onpolicy}


The policy-based agent maintains a target policy network $\pi = \pi_\theta$ for the learner and a set of behavior policy  networks
$\mu = \pi_{\theta^\prime}$ for the actors. The actor parameters $\theta^\prime$ are delayed copies of the learner parameter~$\theta$. To emulate a near on-policy situation $\pi \approx \mu$, we minimize the delay of the parameter passage between the central learner and actors, by hosting both learner/actors on the same machine. 

We compare second-order expansions with two baselines: \textit{first-order} and \textit{zero-order}. For the first-order baseline, we also adopt the PPO technique of clipping: $\text{clip}(\pi(a|x)/\mu(a|x), 1-\epsilon,1+\epsilon)$ in Eq.\,\ref{eq:firstorder} with $\epsilon=0.2$. Clipping the ratio enforces an implicit trust region  with the goal of increased stability \citep{schulman2017proximal}. This technique has been shown to generally outperform a na\"ive explicit constraint, as done in the original TRPO \citep{schulman2015trust}. In Appendix~\ref{appendix:experiment-nearonpolicy}, we detail how we implemented PPO on the asynchronous architecture. Each baseline trains on the entire Atari suite for $400$M frames and we compare the mean/median human-normalized scores.

The comparison results are shown in Figure~\ref{fig:onpolicy}. Please see the median score curves in  Figure~\ref{fig:onpolicy_median} in Appendix~\ref{appendix:experiment-nearonpolicy}. We make several observations: \textbf{(1)} Off-policy corrections are very critical. Going from zero-order (no correction) to first-order improves the performance most significantly, even when the delays between actors and the learner are minimized as much as possible; \textbf{(2)} Second-order correction significantly improves on the first-order baseline. This might be surprising, because when near on-policy, one should expect the difference between additional second-order correction to be less important. This implies that in fully asynchronous architecture, it is challenging to obtain sufficiently on-policy data and additional corrections can be helpful.

\begin{figure}[t]
\includegraphics[width=8cm]{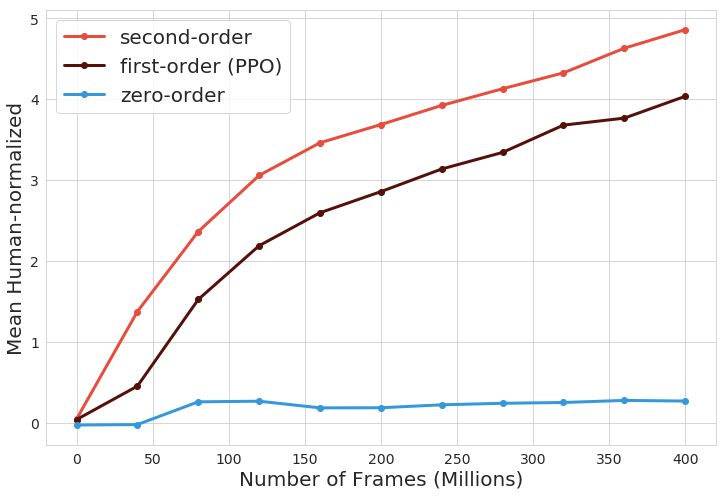}
\centering
\caption{\textbf{Near on-policy optimization.} The x-axis is the number of frames (millions) and y-axis shows the mean human-normalized scores averaged across 57 Atari levels. The plot shows the mean curve averaged across 3 random seeds. We observe that second-order expansions allow for faster learning and better asymptotic performance given the fixed budget on actor steps.}
\label{fig:onpolicy}
\end{figure}

\subsection{Distributed off-policy policy optimization}
\label{ss:gendist}

We adopt the same setup as in Section~\ref{onpolicy}.
To maximize the overall throughput of the agent, the central learner and actors are distributed on different host machines. As a result, both parameter passage from the learner to actors and data passage from actors to the learner could be severely delayed. This creates a natural off-policy scenario with $\pi \neq \mu$.


We compare second-order with two baselines: \emph{first-order} and \emph{V-trace}. The V-trace is used in the original IMPALA agent \citep{espeholt2018impala} and we present its details in Appendix~\ref{appendix:experiment-offpolicypg}. We are interested in how the agent's performance changes as the level of off-policy increases. In practice, the level of off-policy can be controlled and measured as the delay (measured in milliseconds) of the parameter passage from the learner to actors. Results are shown in Figure~\ref{fig:impala}, where x-axis shows the artificial delays (in $\log$ scale) and y-axis shows the mean human-normalized scores after training for $400$M frames. Note that the total delay consists of both artificial delays and inherent delays in the distributed system.

We make several observations: \textbf{(1)} All baseline variants' performance degrades as the delays increase. All baseline off-policy corrections are subject to failures as the level of off-policines increases. \textbf{(2)} While all baselines perform rather similarly when delays are small, as the level of off-policy increases, second-order correction degrades slightly more gracefully than the other baselines. This implies that second-order is a more robust off-policy correction method than other current alternatives.

\label{offpolicy}
\begin{figure}[t]
\includegraphics[width=8cm]{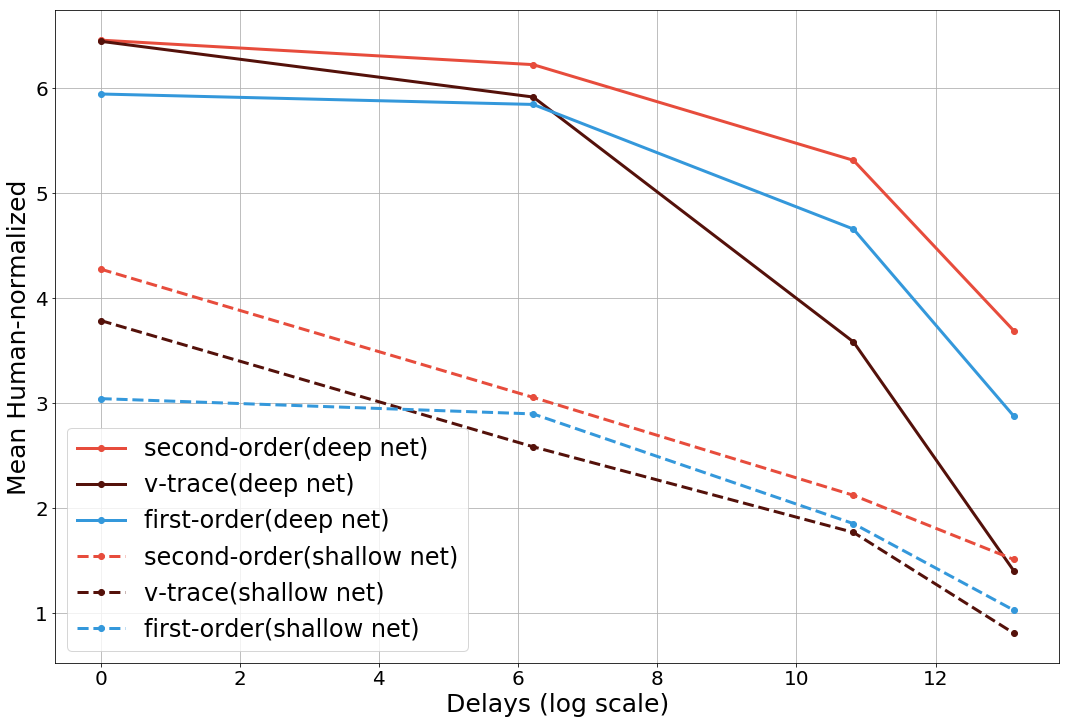}
\centering
\caption{\textbf{Distributed off-policy policy optimization.} The x-axis is the controlled delays between the actors and learner (in log scale) and y-axis shows the mean human-normalized scores averaged across 57 Atari levels after training for 400M frames. Each curve averages across 3 random seeds. Solid curves are results trained with resnets while dashed curves are trained with shallow nets second-order expansions make little difference compared to baselines (V-trace and first-order) when the delays are small. When delays increase, the performance of second-order expansions decay more slowly.}
\label{fig:impala}
\end{figure}
\subsection{Distributed value-based learning}
\label{ss:disvb}

The value-based agent maintains a Q-function network 
$Q_\theta$ for the learner and a set of delayed Q-function networks $Q_{\theta^\prime}$ for the actors. Let $\mathcal{E}$ be an operator such that $\mathcal{E}(Q, \epsilon)$ returns the $\epsilon$-greedy policy with respect to $Q$. The actors generate partial trajectories by executing an $\mu = \mathcal{E}(Q_{\theta^\prime}, \epsilon)$ and send data to a replay buffer. The target policy is greedy with respect to the current Q-function $\pi = \mathcal{E}(Q_\theta, 0)$. The learner samples partial trajectories from the replay buffer and updates parameters by minimizing Bellman errors computed along sampled trajectories. Here we focus on R2D2, a special instance of distributed value-based agent. Please refer to \citet{kapturowski2018recurrent} for a complete review of all algorithmic details of value-based agents such as R2D2.

Across all baseline variants, the learner computes regression targets $Q_{\text{target}}(x,a) \approx Q^\pi(x,a)$ for the network to approximate $Q_\theta(x,a) \approx Q_\text{target}(x,a)$. The targets $Q_\text{target}(x,a)$ are calculated based on partial trajectories under $\mu$ which require off-policy corrections. We compare several correction variants: zero-order, first-order, Retrace \citep{munos2016safe,rowland2019adaptive} and second-order. Please see algorithmic details in Appendix \ref{appendix:secondorderr2d2}.

The comparison results are in Figure \ref{fig:r2d2} where we show the mean scores. We make several observations: \textbf{(1)} second-order correction leads to marginally better performance than first-order and retrace, and significantly better than zero-order. \textbf{(2)} In general, unbiased (or slightly biased) off-policy corrections do not yet perform as well as radically biased off-policy variants, such as \textit{uncorrected-nstep} \citep{kapturowski2018recurrent,rowland2019adaptive}. \textbf{(3)} Zero-order performs the worst --- though it is able to reach super human performance on most games as other variants but then the performance quickly plateaus. See Appendix~\ref{appendix:experiment-valuelearning} for more results.

\label{value}
\begin{figure}[t]
\includegraphics[width=8cm]{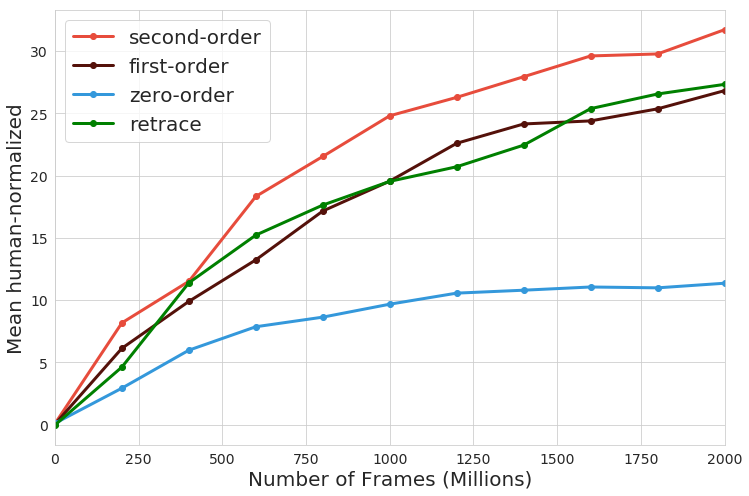}
\centering
\caption{\textbf{Value-based learning with distributed architecture (R2D2).} The x-axis is number of frames (millions) and y-axis shows the mean human-normalized scores averaged across 57 Atari levels over the training of 2000M frames. Each curve averages across 2 random seeds. The second-order correction performs marginally better than first-order correction and retrace, and significantly better than zero-order. See Appendix \ref{appendix:secondorderr2d2} for detailed descriptions of these baseline variants.}
\label{fig:r2d2}
\end{figure}

\section{Discussion and conclusion}
The idea of IS is the core of most off-policy evaluation techniques \citep{precup2000eligibility,harutyunyan_QLambda_2016,munos2016safe}. We showed that Taylor expansions construct approximations to the full IS corrections and hence intimately relate to established off-policy evaluation techniques. 

However, the connection between IS and policy optimization is less straightforward. Prior work focuses on applying off-policy corrections directly to
policy gradient estimators \citep{jie2010connection,espeholt2018impala} instead of the surrogate objectives which generate the gradients. Though standard policy optimization objectives \citep{schulman2015trust,schulman2017proximal} involve IS weights, their link with IS is not made explicit. Closely related to our work is that of \citet{tomczak2019policy}, where they identified such optimization objectives as biased approximations to the full IS objective \citep{metelli2018policy}. We characterized such approximations as the first-order special case of Taylor expansions and derived their natural generalizations.

In summary, we showed that Taylor expansions naturally connect trust-region policy search with off-policy evaluations. This new formulation unifies previous results, opens doors to new algorithms and bring significant gains to certain state-of-the-art deep RL agents.

\paragraph{Acknowledgements.} Great thanks to 
Mark Rowland for insightful discussions during the development of ideas as well as extremely useful feedbacks on earlier versions of this paper. The authors also thank Diana Borsa, Jean-Bastien Grill, Florent Altch\'e,  Tadashi Kozuno,  Zhongwen Xu, Steven Kapturowski, and Simon Schmitt for helpful discussions.



\bibliographystyle{apalike}
\bibliography{refs}

\appendix
\onecolumn

\section{Derivation of results for generalized trust-region policy optimization}
\label{appendix:generalizedtrpo}

\subsection{Controlling the residuals of Taylor expansions}
\label{appendix:generalizedtrpo-residuals}

We summarize the bound on the magnitude of the Taylor expansion residuals of the Q-function as a proposition.
\begin{proposition}

Recall the definition of the Taylor expansion residual of the Q-function from the main text, $E_K \triangleq \sum_{k=K+1}^\infty U_k$. Let $||\cdot||$ be the infinity norm $||A|| \triangleq \max_{x,a}|A(x,a)|$. Let $R_{\text{\rm max}}$ be the maximum reward in the entire MDP, $R_{\text{\rm max}} \triangleq \max_{x,a} |r(x,a)|$. Finally, let $\epsilon \triangleq ||\pi - \mu||_1$. Then 
\begin{align}
    ||E_K|| \leq \left(\frac{\gamma}{1-\gamma}\epsilon\right)^{K+1}\left(1-\frac{\gamma}{1-\gamma}\epsilon\right)^{-1} \frac{R_\text{\rm max}}{1-\gamma}\cdot
\end{align}

\end{proposition}
\begin{proof}
The proof follows by bounding each the magnitude of term $||U_k||$,
\begin{align}
    ||E_K|| = \left\| \sum_{k=K+1}^\infty U_k \right\| &\leq \sum_{k=K+1}^\infty||U_k|| \nonumber \\
    &= \sum_{k=K+1}^\infty \| \gamma^k\left((I-\gamma P^\mu)^{-1}(P^\mu - P^\mu)\right)^{k} \| \cdot ||Q^\mu|| \nonumber \\
    &\leq \sum_{k=K+1}^\infty \left(\frac{\gamma}{1-\gamma}\right)^{k} \epsilon^k \frac{R_{\text{max}}}{1-\gamma} \nonumber \\
    &= \left(\frac{\gamma}{1-\gamma} \epsilon\right)^{K+1} \left(1-\frac{\gamma}{1-\gamma}\epsilon\right)^{-1} \frac{R_{\text{max}}}{1-\gamma}\cdot
    \nonumber
\end{align}
The above derivation shows that once we have $\epsilon < \frac{1-\gamma}{\gamma}\CommaBin$ $||E_K|| \rightarrow 0$ as $K \rightarrow \infty$. In the above derivation, we have applied the bound $||U_k|| \leq \left(\frac{\gamma}{1-\gamma}\right)^k \epsilon^k \frac{R_{\text{max}}}{1-\gamma}\CommaBin$ which will also be helpful in later derivations.
\end{proof}

\subsection{Deriving Taylor expansions of RL objective}
\label{appendix:generalizedtrpo-rlobjective}
Recall that the RHS of Eq.\,\ref{eq:taylorexpansion} are the Taylor expansions of Q-functions $Q^\pi$. By construction, $Q^\pi -Q^\mu = \sum_{k\geq 0} U_k$. Though Eq.\,\ref{eq:taylorexpansion} shows the expansion of the entire vector $Q^\pi$, for optimization purposes, we care about the RL objective from a starting state $x_0$, $J(\pi) = \mathbb{E}_{\pi,a_0\sim \pi(\cdot|x_0),x_0}[ Q^\pi(x_0,a_0)] = \pi_0\transpose Q^\pi$, where $\pi_0 \in \mathbb{R}^{|\mathcal{X}||\mathcal{A}|}$ follows the definition from the main paper $\pi_0(x,a) \triangleq \pi(a|x)\delta_{x=x_0}$.

Now we focus on calculating $L_K(\pi,\mu)$ for general $K\geq 0$. For simplicity, we write $L_k(\pi,\mu)$ as $L_k$ and henceforth we might use these notations interchangeably. Now consider the RHS of Eq.\,\ref{eq:pgobjective}. By definition of the $k$-th order Taylor expansion $L_k$ $(1\leq k\leq K)$ of $J(\pi) - J(\mu)$, we maintain terms where $\pi/\mu - 1$ appears at most $K$ times. Equivalently, in matrix form, we remove the higher order terms of $\pi-\mu$ while only maintaining terms such as $(\pi - \mu)^k, k\leq K$. This allows us to conclude that
\begin{align*}
   \sum_{i=1}^K L_k = \left(\pi_0-\mu_0\right)\transpose \left(Q^\mu + \sum_{k\geq 1}^{K-1} U_k\right) + \mu_0\transpose \left(\sum_{k=1}^K U_k\right).
\end{align*}
Furthermore, we can single out each term
\begin{align*}
   L_k &= (\pi_0-\mu_0)\transpose T_{k-1} + \mu_0\transpose U_k,\ \  k\geq 2 \nonumber \\
    L_1 &= (\pi_0-\mu_0)\transpose Q^\mu + \mu_0\transpose U_1.
\end{align*}

\section{Proof of Theorem~\ref{thm:Taylor}}
\label{appendix:taylorexpansionq}

\begin{proof}
We derive the Taylor expansion of Q-function $Q^\pi$ into different orders of $P^\pi - P^\mu$. For that purpose, we recursively make use of the following matrix equality
\begin{align*}
    (I - \gamma P^\pi)^{-1} = (I-\gamma P^\mu)^{-1} + \gamma (I-\gamma P^\mu)^{-1}(P^\pi - P^\mu)(I-\gamma P^\pi)^{-1},
\end{align*}
which can be derived either from matrix inversion equality or directly verified. Since $Q^\pi = (I-\gamma P^\pi)^{-1} R$, we can use the previous equality to get
\begin{align*}
    Q^\pi &= (I-\gamma P^\mu)^{-1} R + \gamma (I-\gamma P^\mu)^{-1} (P^\pi-P^\mu) (I-\gamma P^\pi)^{-1} R \nonumber \\
    &= Q^\mu + \gamma (I-\gamma P^\mu)^{-1} (P^\pi-P^\mu) Q^\pi. 
\end{align*}
Next, we recursively apply the equality $K$ times,
\begin{align*}
    Q^\pi &= Q^\mu + \sum_{k=1}^K \left(\gamma (I-\gamma P^\mu)^{-1} (P^\pi - P^\mu)\right)^k Q^\mu +  \left(\gamma (I-\gamma P^\mu)^{-1} (P^\pi - P^\mu)\right)^{K+1} Q^\pi.
\end{align*}

Now if $||\pi-\mu||_1 <(1-\gamma)/\gamma$ then we can bound the  sup-norm in of the above term as \[\|\gamma (I-\gamma P^\mu)^{-1} (P^\pi - P^\mu)\|_\infty=\frac{\gamma}{1-\gamma}||\pi-\mu||_1 <1,\] thus the $(K+1)$-th order residual term vanishes when $K\rightarrow\infty$. As a result, the limit $K\rightarrow \infty$ is well defined and we deduce
\begin{align*}
    Q^\pi &= \sum_{k=0}^\infty \left(\gamma (I-\gamma P^\mu)^{-1} (P^\pi - P^\mu)\right)^k Q^\mu. 
\end{align*}
\end{proof}

\section{Proof of Theorem \ref{thm:monotonic}}
\label{sec:proofmonotonic}
\begin{proof}
To derive the monotonic improvement theorem for generalized TRPO, it is critical to bound $\sum_{k=K+1}^\infty L_k$. We achieve this by simply bounding each term separately. Recall that from Appendix \ref{appendix:generalizedtrpo-residuals} we have $||U_k|| \leq \left(\frac{\gamma}{1-\gamma}\right)^k \epsilon^k R_\text{max}$. Without loss of generality, we first assume $R_\text{max}=1-\gamma$ for ease of derivations.
\begin{align*}
   |L_k| \leq \epsilon ||U_{k-1}|| + ||U_k|| \leq \epsilon \left(\frac{\gamma}{1-\gamma}\right)^{k-1} \epsilon^{k-1} +  \left(\frac{\gamma}{1-\gamma}\right)^{k} \epsilon^{k} = \left(\frac{\gamma}{1-\gamma}\right)^{k-1} \frac{1}{1-\gamma} \epsilon^k.
\end{align*}
This leads to a bound over the residuals
\begin{align*}
   \left|\sum_{k=K+1}^\infty L_k\right| \leq \sum_{k=K+1}^\infty |L_k| \leq \sum_{k=K+1}^\infty \left(\frac{\gamma}{1-\gamma}\right)^{k-1} \frac{1}{1-\gamma} \epsilon^k = \frac{1}{\gamma}\left(1-\frac{\gamma}{1-\gamma}\epsilon\right)^{-1} \left(\frac{\gamma\epsilon}{1-\gamma}\right)^{K+1}.
\end{align*}
Since we have the equality $J(\pi) = J(\mu) + \sum_{k=1}^\infty L_k$ for $||\pi-\mu||_1 \leq \epsilon < \frac{1-\gamma}{\gamma}\CommaBin$ we can deduce the following monotonic improvement, 
\begin{align}
    J(\pi) \geq J(\mu) + \sum_{k=1}^K L_k - \frac{1}{\gamma}\left(1-\frac{\gamma}{1-\gamma}\epsilon\right)^{-1} \left(\frac{\gamma\epsilon}{1-\gamma}\right)^{K+1}.
\end{align}
To write the above statement in a 
compact way, we define the gap \[G_K = \frac{1}{\gamma}\left(1-\frac{\gamma}{1-\gamma}\epsilon\right)^{-1} \left(\frac{\gamma\epsilon}{1-\gamma}\right)^{K+1}.\] To derive the result for general $R_\text{max}$, note that the gap $G_K$ has a linear dependency on $R_\text{max}$. Hence the general gap is \[G_K \triangleq \frac{1}{\gamma(1-\gamma)}\left(1-\frac{\gamma}{1-\gamma}\epsilon\right)^{-1} \left(\frac{\gamma\epsilon}{1-\gamma}\right)^{K+1}R_\text{max},\] which gives produces the monotonic improvement result (Eq.\,\ref{eq:monotonic}) stated in the main paper.
\end{proof}

\section{Proof of Theorem \ref{thm:variancereduction}}
\label{sec:proofvariancereduction}
\begin{proof}
It is known that for $K=1$, replacing $Q^\mu(x,a)$ by $A^\mu(x,a)$ in the estimation can potentially reduce variance \citep{schulman2015trust,schulman2017proximal} yet keeps the estimate unbiased. Below, we show that in general, replacing $Q^\pi(x,a)$ by $A^\pi(x,a)$ renders the estimate of $L_K(\pi,\mu)$ unbiased for general $K\geq 1$.

As shown above and more clearly in Appendix \ref{appendix:taylorexpansionrl}, $L_K(\pi,\mu)$ can be written as 
\begin{align}
    L_K(\pi,\mu) = \mathbb{E}_{(x^{(i)},a^{(i)})_{1\leq i\leq K}}\left[\prod_{i=1}^K \left(\frac{\pi(a_i|x_i)}{\mu(a_i | x_i)} -1 \right) Q^{\mu}(x_K,a_K)\right].
\end{align}
Note that for clarity, in the above expectation, we omit an explicit sequence of discounted visitation distributions (for detailed derivations of this sequence of visitation distributions, see Appendix \ref{appendix:taylorexpansionrl}). Next,
we leverage the conditional expectation with respect to $(x^{(i)},a^{(i)}),1\leq i\leq K-1$ to yield 
\begin{align}
    L_K(\pi,\mu) &= \mathbb{E}_{(x^{(i)},a^{(i)})_{1\leq i\leq K-1}}\left[\prod_{i=1}^{K-1} \left(\frac{\pi(a_i|x_i)}{\mu(a_i | x_i)} -1 \right) \mathbb{E}_{(x_K,a_K)}\left[\left(\frac{\pi(a_K|x_K)}{\mu(a_K|x_K)}-1\right) Q^{\mu}(x_K,a_K)\right]\right] \nonumber \\
    &= \mathbb{E}_{(x^{(i)},a^{(i)})_{1\leq i\leq K-1}}\left[\prod_{i=1}^{K-1} \left(\frac{\pi(a_i|x_i)}{\mu(a_i | x_i)} -1 \right) \mathbb{E}_{(x_K,a_K)}\left[\left(\frac{\pi(a_K|x_K)}{\mu(a_K|x_K)}-1\right) A^{\mu}(x_K,a_K)\right] \right]\nonumber \\
 &= \mathbb{E}_{(x^{(i)},a^{(i)})_{1\leq i\leq K}}\left[\prod_{i=1}^K \left(\frac{\pi(a_i|x_i)}{\mu(a_i | x_i)} -1 \right) A^{\mu}(x_K,a_K)\right].
\end{align}
The above derivation shows that indeed, replacing $Q^\mu(x,a)$ by $A^\mu(x,a)$ does not change the value the expectation, while potentially reducing the variance of the overall estimation. 
\end{proof}

\section{Proof of Theorem~\ref{thm:off-policy}}
\label{sec:proofreturn}

\begin{proof}
From the definition of the \textit{return off-policy evaluation operator} $\mathcal{R}_1^{\pi,\mu}$, we have 
\begin{align*}
\mathcal{R}_1^{\pi,\mu}Q &= Q+\left(I-\gamma P^{\mu}\right)^{-1}\left(r+\gamma P^{\pi}Q-Q\right) \\
&= \left(I-\gamma P^{\mu}\right)^{-1}\left(r+\gamma P^{\pi}Q-Q + \left(I-\gamma P^\mu\right) Q\right) \\
&= \left(I-\gamma P^{\mu}\right)^{-1} r+\gamma \left(I-\gamma P^{\mu}\right)^{-1}\left( P^{\pi}- P^\mu\right) Q \\
&= Q^\mu+ \gamma \left(I-\gamma P^{\mu}\right)^{-1}\left(P^{\pi}-P^\mu\right) Q.
\end{align*}
Thus $Q\mapsto \mathcal{R}_1^{\pi,\mu}Q$ is a linear operator, and 
\begin{align*}
\left(\mathcal{R}_1^{\pi,\mu}\right)^2 Q &= Q^\mu+ \gamma \left(I-\gamma P^{\mu}\right)^{-1}\left(P^{\pi}-P^\mu\right) \left(\mathcal{R}_1^{\pi,\mu}\right)^2 Q\\
&= Q^\mu+ \gamma \left(I-\gamma P^{\mu}\right)^{-1}\left(P^{\pi}-P^\mu\right) Q^\mu +  \left[\gamma \left(I-\gamma P^{\mu}\right)^{-1}\left(P^{\pi}-P^\mu\right)\right]^2 Q.\\
\end{align*}
Applying this step $K$ times, we deduce
\begin{align*}
\left(\mathcal{R}_1^{\pi,\mu}\right)^K Q &= Q^\mu+ \sum_{k=1}^{K-1} \left[ \gamma \left(I-\gamma P^{\mu}\right)^{-1}\left(P^{\pi}-P^\mu\right)\right]^k Q^\mu+ 
\left[ \gamma \left(I-\gamma P^{\mu}\right)^{-1}\left(P^{\pi}-P^\mu\right) \right]^{K} Q.
\end{align*}
Applying the above operator to $Q^\mu$ we deduce that 
\begin{align*}
\left(\mathcal{R}_1^{\pi,\mu}\right)^K Q^\mu &= Q^\mu+ \sum_{k=1}^{K} \underbrace{\left[ \gamma \left(I-\gamma P^{\mu}\right)^{-1}\left(P^{\pi}-P^\mu\right)\right]^k Q^\mu}_{=U_k},
\end{align*}
which proves our claim.
\end{proof}

\section{Alternative derivation for Taylor expansions of RL objective}
\label{appendix:taylorexpansionrl}
In this section, we provide an alternative derivation of the Taylor expansion of the RL objective. Let $\pi_t / \mu_t = 1 + \epsilon_t$. In cases where $\pi \approx \mu$ (e.g., for the trust-region case), $\epsilon \approx 0$. To calculate $J(\pi)$ using data from $\mu$, a natural technique is employ importance sampling (IS),
\begin{align*}
     J(\pi) &= \mathbb{E}_{\mu,x_0}\left[\left(\prod_{t=0}^\infty \frac{\pi_t}{\mu_t}\right) \sum_{t=0}^\infty r_t \gamma^t\right] 
    = \mathbb{E}_{\mu,x_0}\left[\left(\prod_{t=0}^\infty (1+\epsilon_t)\right) \sum_{t=0}^\infty \gamma^t r_t\right]\cdot 
\end{align*}
To derive Taylor expansion in an \textit{intuitive} way, consider expanding the product $\prod_{t=0}^\infty (1+\epsilon_t)$, assuming that this infinite product is finite. Assume all $\epsilon_t\leq \epsilon$ with some small $\epsilon>0$. A second-order Taylor expansion is
\begin{align}
    \prod_{t=0}^\infty \left(1+\epsilon_t\right) = 1 + \sum_{t=0}^\infty \epsilon_t + \sum_{t=0}^\infty\sum_{t^\prime > t}^\infty \epsilon_t\epsilon_{t^\prime} + O(\epsilon^3). 
\end{align}
Now, consider the term associated with $\sum_{t=0}^\infty \epsilon_t$, 
\begin{align}
    \mathbb{E}_{\mu,x_0}\left[\sum_{t=0}^\infty \epsilon_t \sum_{t=0}^\infty \gamma^t r_t\right] &= \mathbb{E}_{\mu,x_0}\left[\sum_{t=0}^\infty \epsilon_t \sum_{t^\prime=t}^\infty r_t\gamma^{t^\prime} \right] \nonumber \\
&= \mathbb{E}_{\mu,x_0}\left[\sum_{t=0}^\infty \epsilon_t \sum_{t^\prime=t}^\infty r_t\gamma^{t} \gamma^{t^\prime-t} \right] \nonumber \\
&= \mathbb{E}_{\mu,x_0}\left[\sum_{t=0}^\infty \epsilon_t Q^\mu\left(x_t,a_t\right) \gamma^{t} \right] \nonumber \\
&= (1-\gamma) \mathop{\mathbb{E}}_{ \substack{x,a\sim d_\gamma^\mu(\cdot,\cdot|x_0,a_0,0) \\a_0 \sim \mu\left(\cdot|x_0\right)}} \left[\left(\frac{\pi(a|x)}{\mu(a|x)}-1\right) Q^\mu(x,a)\right].
\end{align}
Note that in the last equality, the $\gamma^t$ factor is absorbed into the discounted visitation distribution $d_\gamma^\mu(\cdot,\cdot|x_0,a_0,0)$. It is then clear that this term is exactly the first-order expansion $L_1(\pi,\mu)$ shown in the main paper. 

Similarly, we could derive the second-order expansion by studying the term associated with $\sum_{t=0}^\infty\sum_{t^\prime > t}^\infty \epsilon_t\epsilon_{t^\prime}$. 
\begin{align}
\mathbb{E}_{\mu,x_0}\left[\sum_{t=0}^\infty\sum_{t^\prime > t}^\infty \epsilon_t\epsilon_{t^\prime} \sum_{t=0}^\infty \gamma^t r_t\right] &= \mathbb{E}_{\mu,x_0}\left[\sum_{t=0}^\infty\sum_{t^\prime > t}^\infty \epsilon_t\epsilon_{t^\prime} \sum_{\tau=t^\prime}^\infty r_\tau\gamma^\tau\right] \nonumber \\
&= \mathbb{E}_{\mu,x_0}\left[\sum_{t=0}^\infty\sum_{t^\prime > t}^\infty \epsilon_t\epsilon_{t^\prime} \sum_{\tau=t^\prime}^\infty r_\tau\gamma^{\tau-t^\prime} \gamma^{t}\gamma^{t^\prime-t}\right] \nonumber \\
&= \mathbb{E}_{\mu,x_0}\left[\sum_{t=0}^\infty\sum_{t^\prime > t}^\infty \epsilon_t\epsilon_{t^\prime} Q^\mu(x_{t^\prime},a_{t^\prime}) \gamma^{t}\gamma^{t^\prime-t}\right] \nonumber \\
&= \frac{(1-\gamma)^2}{\gamma} \mathop{\mathbb{E}}_{ \substack{x,a\sim d_\gamma^\mu(\cdot,\cdot|x_0,a_0,0) \\a_0 \sim \mu(\cdot|x_0)\\ x^\prime,a^\prime\sim d_\gamma^\mu(\cdot,\cdot|x,a,1) }} \left[\left(\frac{\pi(a|x)}{\mu(a|x)}-1\Big) \Big(\frac{\pi(a^\prime|x^\prime)}{\mu(a^\prime|x^\prime)}-1\right) Q^\mu(x^\prime,a^\prime)\right].
\end{align}
Note that similar to the first-order expansion, the discount factor $\gamma^{t}\gamma^{t^\prime-t}$ is absorbed into the discounted visitation distribution $d_\gamma^\mu(\cdot,\cdot|x_0,a_0,0)$ and $d_\gamma^\mu(\cdot,\cdot|x,a,1)$ respectively. Here note that the second discounted visitation distribution is $d_\gamma^\mu(\cdot,\cdot|x,a,1)$ instead of $d_\gamma^\mu(\cdot,\cdot|x,a,0)$ --- this is because $t^\prime-t \geq 1$ by construction and we need to sample the second state \textit{conditional} on the time difference to be $\geq 1$. The above is exactly the second-order expansion $L_2(\pi,\mu)$.

By a similar argument, we can derive expansion for all higher-order expansion by considering the term associated with $\sum_{t_1=0}^\infty \sum_{t_2 > t_1}^\infty ... \sum_{t_K > t_{K-1}}^\infty \epsilon_1\epsilon_2\dots\epsilon_K$. This would introduce $K$ discounted visitation distributions $d_\gamma^\mu(\cdot,\cdot|x_0,a_0,0)$ and $d_\gamma^\mu(\cdot,\cdot|x_k,a_k,1),1\leq k\leq K$. 

The above derivation also illustrates how these higher-order terms can be estimated in practice. For the $k$-th order, given a trajectory under $\mu$, sequentially sample $K$ time difference $\Delta t_k$ along the trajectory, where $t_1\sim \text{Geometric}(1-\gamma)$. For $k\geq 2$, $t_k \sim \text{Geometric}(1-\gamma)$ while conditional on $\Delta t_k \geq 1$. Then define the time $t_k = \sum_{i \leq k}\Delta t_{i}$. Let $x_i = x_{t_i}$ and $a_i = a_{t_i}$, Then, a one sample  estimate is 
\begin{align}
    \prod_{i=1}^K \left(\frac{\pi(a_i|x_i)}{\mu(a_i | x_i)} -1 \right) Q^{\mu}(x_K,a_K).
\end{align}

\subsection{Connection between off-policy evaluation and generalized advantage estimation (GAE)}
\label{appendix:gae}
\textit{Generalized advantage estimation} (GAE, \citealt{schulman2015high}) is a technique for advantage estimation. According to \citet{schulman2015high,schulman2017proximal}, GAE trades-off bias and variance in the advantage estimation and can boost the performance of downstream policy optimization. On the other hand, off-policy evaluation operators \citep{harutyunyan_QLambda_2016,munos2016safe} are dedicated to evaluations of Q-function $Q^\pi(x,a)$. What are the connections between these approaches?

The actor-critic algorithm that uses GAE maintains a policy $\pi(a|x)$ and value function $V_\phi(x)$ with parameter $\phi$. Data are collected on-policy, i.e., $\mu=\pi$. Let $\hat{A}_{\text{GAE}}(x,a)$ be the GAE estimation for $(x,a)$. Naturally, GAE can be interpreted as first carrying out a Q-function estimation $\hat{Q}(x,a)$ and then subtracting the baseline
\begin{align}
    \hat{A}_\text{GAE}(x,a) \triangleq \hat{Q}(x,a) - V_\phi(x).
\end{align}
Now we show that the Q-function estimation $\hat{Q}(x,a)$ can be interpreted as applying the $Q(\lambda)$ operator to an initial Q-function estimate. Here importantly, to make the connection exact, we assume the initial Q-function estimate to be bootstrapped from the value function $Q_\text{init}(x,a) \triangleq V_\phi(x)$. To sum up,
\begin{align}
    \hat{A}_\text{GAE}(x,a) = \left[ \mathcal{R}_{c=\lambda}^{\pi,\pi} Q_\text{init} \right](x,a)  - V_\phi(x),
\end{align}
where $\mathcal{R}_{c=\lambda}^{\pi,\mu}$ refers to the evaluation operator with trace coefficients $c(x,a) = \lambda$. Finally, the evaluation operator is replaced by sample estimates in practice. 
From the above, we see that there is a link between advantage estimation (i.e., GAE) with policy evaluation (i.e., the $Q(\lambda)$ operator).

\section{Second-order expansions for value-based algorithms}
\label{appendix:secondorderr2d2}
In this section, we provide algorithmic details on value-based algorithms in our experiments. The application of Taylor expansions allow us to derive the expansion for RL objective, which is useful in policy-optimization where algorithms maintain a parameterized policy $\pi_\theta$. Taking one step back, Taylor expansion can be used for policy evaluation as well, and can be useful in algorithms where Q-functions (value functions) are parameterized $Q_\theta$ where the policy is implicitly defined (e.g., $\epsilon$-greedy). In our experiments, we take R2D2 \citep{kapturowski2018recurrent} as the baseline algorithm. Below, we briefly introduce the algorithmic procedure of R2D2 and present the Taylor expansion variants.

\paragraph{Basic components.} The baseline R2D2 maintains a Q-function $Q_\theta(x,a)$ parameterized by a neural network $\theta$. The central learner maintains an updated parameter $\theta$ and distributed actors maintain slightly delayed copies $\theta_{\text{old}}$. Distributed actors collect data using behavior policy $\mu$, defined as $\epsilon$-greedy with respect to $Q_{\theta_{\text{old}}}(x,a)$. The target policy $\pi$ is defined as greedy with respect to $Q_\theta(x,a)$. Actors send data to a replay buffer, and the learner samples partial trajectories $(x_t,a_t,r_t)_{t=1}^T$ from the buffer and computes updates to the parameter $\theta$. In particular, the learner calculates regression targets $Q_{\text{target}}(x_t,a_t)$ and the 
Q-function is updated via $\theta\leftarrow \theta - \alpha \cdot \nabla_\theta (Q_\theta(x_t,a_t)-Q_{\text{target}}(x_t,a_t))^2$ with learning rate $\alpha > 0$.

\paragraph{Algorithmic variants.}  Algorithmic variants of R2D2 differ in how they compute the targets  $Q_{\text{target}}(x,a)$. A useful unified view provided by \citet{rowland2019adaptive} is that $Q_{\text{target}}(x,a)$ aims to approximate $Q^\pi(x,a)$ such that $Q_\theta(x,a) \rightarrow Q^\pi(x,a)$ during the update. 

Along sampled trajectories, we recursively calculate the targets $Q_{\text{target}}(x_t,a_t),1\leq t\leq T$ based on recipes of different variants. Below are a few alternatives we  evaluated in our experimenst, where we e.g., use $\hat{Q}_{\text{zero}}(x_t,a_t)$ to represent $Q_{\text{target}}(x_t,a_t)$ for the zero-order baseline.
\begin{itemize}
    \item \textbf{Zero-order:} $\hat{Q}_{\text{zero}}(x_t,a_t) \triangleq r_t + \gamma \hat{Q}_{\text{zero}}(x_{t+1},a_{t+1})$
    \item \textbf{First-order:} $\hat{Q}_{\text{first}}(x_t,a_t) \triangleq r_t + \gamma\left(\mathbb{E}_{\pi}\left[Q_\theta(x_{t+1},\cdot)\right] - Q_\theta(x_{t+1},a_{t+1})\right) + \gamma \hat{Q}_{\text{first}}(x_{t+1},a_{t+1})$
    \item \textbf{Second-order:} $\hat{Q}_{\text{second}}(x_t,a_t) \triangleq r_t + \gamma\left(\mathbb{E}_{\pi}\left[\hat{Q}_\text{first}^\prime(x_{t+1},\cdot)\right] - \hat{Q}_\text{first}(x_{t+1},a_{t+1})\right) + \gamma \hat{Q}_{\text{second}}(x_{t+1},a_{t+1})$
    \item \textbf{Retrace:} $\hat{Q}_{\text{retrace}}(x_t,a_t) \triangleq r_t + \gamma c_{t+1}\left(\mathbb{E}_{\pi}\left[Q_\text{retrace}(x_{t+1},\cdot)\right] - Q_\text{retrace}(x_{t+1},a_{t+1})]\right) + \gamma c_{t+1}\hat{Q}_{\text{retrace}}(x_{t+1},a_{t+1})$.
\end{itemize}
For retrace, we set the trace coefficient $c_t \triangleq \lambda \cdot \min\left\{\frac{\pi(a_t|x_t)}{\mu(a_t|x_t)},1\right\}$ following \citet{munos2016safe}. All baselines bootstrap $\hat{Q}_{\text{target}}(x_T,a_T) = Q_\theta(x_t,a_T)$ from the Q-function network for the last state-action pair. 

As shown above, the zero-order baseline reduces to discounted sum of returns (plus a bootstrap value at the end of the trajectory). The first-order adopts the $Q(\lambda),\lambda=1$ recursive update rule. The second-order corresponds to applying $Q(\lambda),\lambda=1$ twice to the partial trajectory---in particular, this corresponds to replacing the Q-function baseline $Q_\theta(x,a)$ by first-order approximations $Q_{\text{first}}(x,a)$. For the above, we define $\hat{Q}_{\text{first}}^\prime(x_t,a) \triangleq \mathbb{I}_{a=a_t} \hat{Q}_{\text{first}}(x_t,a) + (1 - \mathbb{I}_{a=a_t}) Q_\theta(x_t,a)$ where $\mathbb{I}$ is the indicator function. This ensures that the expectations are well defined in the recursive updates.

As discussed in the main paper, it is not always necessarily optimal to carry out exact first/second-order correction, it might be potentially beneficial to strike a balance in between for bias-variance trade-off. To this end, we define the ultimate second-order target as $\hat{Q}_{\text{target-final}} = \hat{Q}_{\text{first}} + \eta (\hat{Q}_{\text{second}} - \hat{Q}_{\text{first}})$ for $\eta \geq 0$. 

See Figure \ref{fig:r2d2} and Figure \ref{fig:r2d2all} for the comparison results of these algorithmic variants. Further hyper-parameter details are specified in Appendix \ref{appendix:experiment-valuelearning}.

\section{Additional experimental details and results}
\label{appendix:experiment}
\subsection{Random MDP}
\label{appendix:experiment-randommdp}

The random MDP is identified by the number of states $|\mathcal{X}|$ and actions $|\mathcal{A}|$. The transitions $p(x^\prime|x,a)$ are generated as samples from a Dirichlet distribution. The reward function $r(x,a)$ is generated as a Dirac, sampled uniformly at random from $[-1,1]$. The discount factor is set to  $\gamma\triangleq 0.9$. The results in Figure~\ref{fig:smallmdp} are averaged over 10 MDPs.

We randomly fix a target policy $\pi$ and randomly sample another behavior policy $\mu$ in the vicinity of $\pi$ such that $||\pi-\mu||_1 \leq \epsilon,$ for some fixed $\epsilon>0$. Effectively, $\epsilon$ controls the off-policiness measured as the difference between $\pi$ and $,\mu$. When using the reward estimate $\hat{R}$ to compute the Q-function estimate, trajectories are generated under the behavior policy $\mu$. The reward estimate is initialized to be zeros $\hat{R}(x,a) \triangleq 0$ for all $x$ and $a$. Since the rewards are deterministic, we have have that when $(x,a)$ is encountered then $\hat{R}(x,a) \leftarrow R(x,a)$.

\subsection{Evaluation of distributed experiments}
\label{appendix:experiment-evaluation}

For this part, the evaluation environments is the entire suite of Atari games \citep{bellemare2013arcade} consisting of $57$ levels. Since each level has very different reward scale and difficulty, we report human-normalized scores for each level, calculated as $z_i = (r_i - o_i) / (h_i - o_i)$, where $h_i$ and $o_i$ are the performances of human and a random policy on level $i$ respectively. 

For all experiments, we report summarizing statistics of the human-normalized scores across all levels. For example, at any point in training, the mean human-normalized score is the mean statistic across $z_i,1\leq i\leq 57$.

\begin{figure}[t]
\includegraphics[width=6cm]{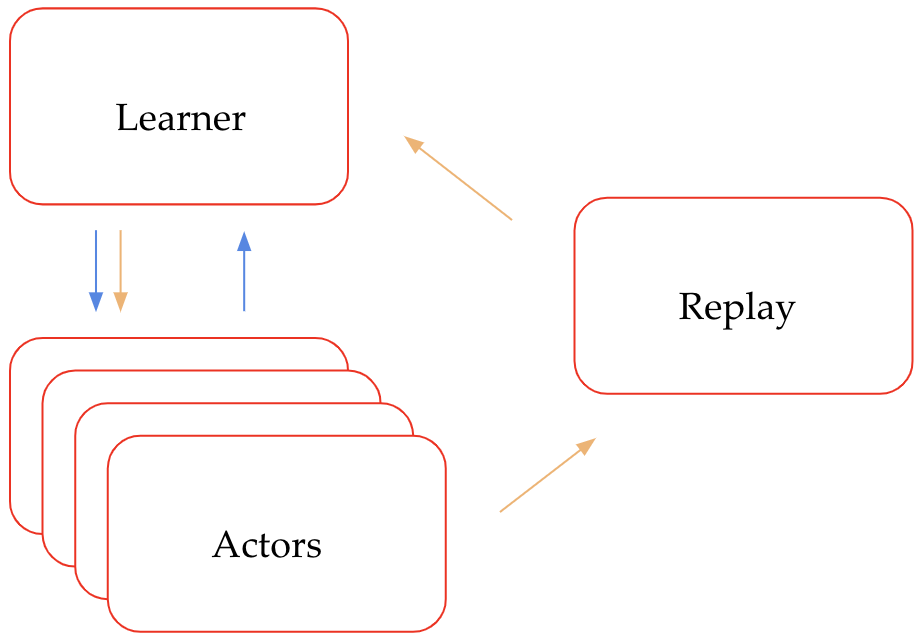}
\centering
\caption{Architecture of distributed agents. Agents differ by the topology, i.e., how actors/learner/replay pass data/parameters between them. The above architecture summarizes common setups such as IMPALA  \citep{espeholt2018impala} as blue arrows and R2D2 \citep{kapturowski2018recurrent} as orange arrows.}
\label{fig:architecture}
\end{figure}

\subsection{Details on distributed algorithms} 
\label{appendix:experiment-architecture}

Distributed algorithms have led to significant performance gains on challenging domains \citep{nair2015massively,mnih2016asynchronous,babaeizadeh2016reinforcement,barth2018distributed,horgan2018distributed}. Here, our focus is on recent state-of-the-art algorithms. In general, distributed agents consist of one central learner, multiple actors and optionally a replay buffer. The central learner maintains a parameter copy $\theta$ and updates parameters based on sampled data. Multiple actors each maintaining a slighted delayed parameter copy $\theta_{\text{old}}$ and interact with the environment to generate partial trajectories. Actors synchronize parameters from the learner periodically. 

Algorithms differ by how are data and parameters passed between each component. We focus on two types of state-of-the-art scalable topologies: \textbf{Type I} adopts  IMPALA-typed architecture (\citealp{espeholt2018impala}; see blue arrows in Figure \ref{fig:architecture} in Appendix \ref{appendix:experiment}), data are directly passed from actors to the learner. See Section \ref{onpolicy} and Section \ref{ss:gendist}; \textbf{Type II}. adopts R2D2-typed architecture (\citealt{kapturowski2018recurrent}, see orange arrows in Figure \ref{fig:architecture} in Appendix~\ref{appendix:experiment}), data are sent from actors to a replay, and later sampled according to priorities to the learner \citep{horgan2018distributed}. 

\subsection{Details on TayPO-$2$ for policy optimization}
\label{appendix:experiment-taypo2}

\paragraph{Discussion on the first-order objective.} By construction, the first-order objective (Eq.\,\ref{eq:firstorder}) samples states with a discounted visitation distribution. Though such an objective is conducive to theoretical analysis, it is too conservative in practice. Indeed, the practical objective is usually undiscounted $\approx \mathbb{E}_{x_0,a_0}[\sum_{t=0}^{T_e} r_t]$ where $T_e$ is an artificial threshold of the episode length. Therefore, in practice, the state $x$ is sampled `uniformly' from generated trajectories, i.e., without the discount factor~$\gamma^t$.

\paragraph{Discussion on the TayPO-$2$ objective.} For the second-order objective (Eq.\,\ref{eq:secondorder}), recall that we sample two state-action pairs $(x,a),(x^\prime,a^\prime)$. In practice, we sample $(x,a)$ uniformly (without discount) as the first-order objective and sample $(x^\prime,a^\prime)$ with discount factors $\gamma^{\Delta t}$ where $\Delta t$ is the time difference between $x^\prime$ and $x$. This is to ensure that we have a comparable loss function $\hat{L}_2(\pi,\mu)$ compared to the first-order $\hat{L}_1(\pi,\mu)$.

\paragraph{Further practical considerations.} In practice, loss functions are computed on partial trajectories $(x_t,a_t)_{t=1}^T$ with length~$T$. Though, theoretically, evaluating $\hat{L}_2(\pi_\theta,\mu)$ requires generating time steps from a geometric distribution $\text{Geometric}(1-\gamma)$ which can exceed the length $T$, in practice, we apply the truncation at $T$. In addition, we evaluate $\hat{L}_2(\pi_\theta,\mu)$ by enumerating over the entire (truncated) trajectory instead of sampling time steps. This comes at several trade-offs: enumerating the trajectory require ${\cal O}(T^2)$ computations while sampling can reduce this complexity to ${\cal O}(T)$; enumerating over all steps could reduce the variance by pooling all data of the trajectory, but could also increase the variance due to correlations of state-action pairs on a single trajectory. In practice, we find enumerating all steps along the trajectory works well.

\subsection{Near on-policy policy optimization}
\label{appendix:experiment-nearonpolicy}

\paragraph{Additional results.} The additional results on the Atari suite are in Figure~\ref{fig:onpolicy_median}, where, we show the median human normalized scores during training. We notice that the second-order still steadily outperforms other baselines.
\begin{figure}[t]
\includegraphics[width=8cm]{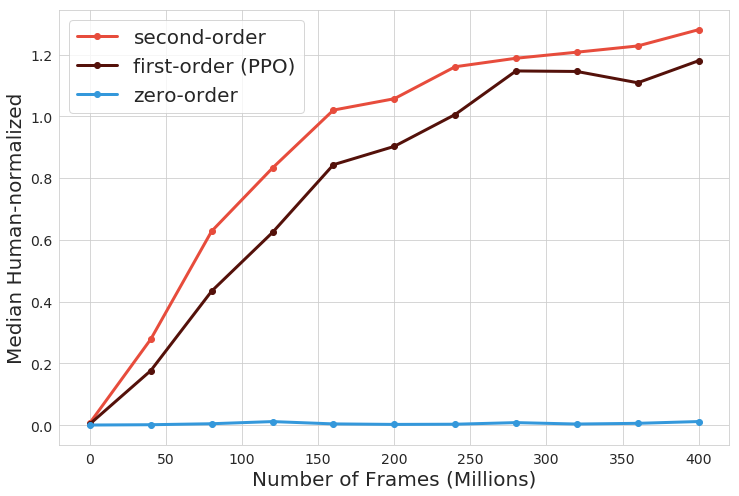}
\centering
\caption{Near on-policy optimization. The x-axis is the number of frames (millions) and y-axis shows the median human-normalized scores averaged across 57 Atari levels. The plot shows the mean curve averaged across 3 random seeds. We observe that second-order expansions allow for faster learning and better asymptotic performance given the fixed budget on actor steps.}
\label{fig:onpolicy_median}
\end{figure}

\paragraph{Discussion on proximal policy optimization (PPO) implementation.} By design, PPO \citep{schulman2017proximal} alternates between data collection and policy update. The data are always collected under $\mu = \pi$ and the new policy gets updated via several gradient steps on the same batch of data. In practice, such a `fully-synchronized` implementation is not efficient because it does not leverage a distributed computational architecture. To improve the implementation, we modify the original algorithm and adapt it to an asynchronous setting. To this end, several changes must be made to the algorithm.

\begin{itemize}
    \item The data are collected with actor policy $\mu$ instead of the previous policy.
    \item The number of gradient descent per batch is one instead of multiple, to balance the data throughput from the actor.
\end{itemize}

\paragraph{Details on computational architecture.} For the near on-policy optimization experiments, we set up an agent with an algorithmic architecture similar to that of IMPALA \citep{espeholt2018impala}. In order to minimize the delays between actors and the central learner, we schedule all components of the algorithms on a single host machine. The learner uses a single TPU for fast inference and computation, while the actors use CPUs for fast batched environment rollouts.

We apply a small network similar to \citet{mnih2016asynchronous}, please see Appendix \ref{appendix:experiment-offpolicypg} for detailed descriptions of the architecture.

Following the conventional practice of training on Atari games \citep{mnih2016asynchronous}, we clip the reward between $[-1,1]$. The learner applies a discount $\gamma=0.99$ to calculate value function targets. The total loss function is a linear combination of policy loss $L_\text{policy}$, value function loss $L_{\text{value}}$ and entropy regularization $L_{\text{entropy}}$, i.e., $L\triangleq L_\text{policy} + c_v L_\text{value} + c_e L_{\text{entropy}}$ where $c_v\triangleq0.5$ and $c_e\triangleq0.01$. All missing details are the same as the hyper-parameter setup of the IMPALA architecture to be introduced below.

The networks are optimized with a RMSProp optimizer \citep{tieleman2012lecture} with the learning rate $\alpha \triangleq 10^{-3}$.

\subsection{Distributed off-policy policy optimization}
\label{appendix:experiment-offpolicypg}

\paragraph{V-trace implementations.} V-trace is a strong baseline for correcting off-policy data \citep{espeholt2018impala}. Given a partial trajectory $(x_t,y_t,r_t)_{t=1}^T$, let $\rho_t \triangleq \min\{\bar{\rho},\pi(a_t|x_t) / \mu(a_t|x_t)\}$ be the truncated IS ratio. Let $V_\phi(x)$ be a value function baseline. Define $\delta_t V \triangleq \rho_t (r_t + \gamma V(x_{t+1}) - V(x_t))$ be a temporal difference. V-trace targets are calculated recursively as
\begin{align}
    v(x_t) \triangleq V(x_t ) + \delta_t V + \gamma c_t\left(v(x_{t+1}) - V(x_{t+1})\right),
\end{align}
where $c_t \triangleq \min\{\bar{c}, \pi(a_t|x_t) / \mu(a_t,x_t)\}$ is the trace coefficient. The value function baseline is then trained to approximate these targets $V_\phi(x) \approx v(x)$.

The policy gradient is corrected by clipped IS ratio as well. The policy parameter $\theta$ is updated using the gradient
\begin{align}
    \min\{\bar{\rho}, \pi(a_t|x_t) / \mu(a_t | x_t)\} \nabla_\theta \log \pi(a_t|x_t) \hat{a}_t,
\end{align}
where the advantage estimates are $\hat{a}_t \triangleq  r_t + \gamma v(x_{t+1}) - v(x_t)$ and the derivative $\nabla_\theta$ is taken with respect to the learner parameter $\pi(a|x) = \pi_\theta(a|x)$. Following the original setup \citep{espeholt2018impala}, we set $\bar{\rho} \triangleq \bar{c} \triangleq 1$.

\paragraph{Hyper-parameters for Taylor expansions.} The Taylor expansion variants (including first-order and second-order expansions) all adopt the surrogate loss functions introduced in the main text. The second-order expansion requires a hyper-parameter $\eta$ which we set to $\eta=1$.

The value function targets are estimated as uncorrected cumulative returns, computed recursively $v(x_t) = r_t + \gamma v(x_{t+1})$ and then the value function baseline is trained to $V_\phi(x) \approx v(x)$. Though adopting more complex estimation techniques such as GAE \citep{schulman2015high} could potentially improve the accuracy of the bootstrapped values.

\paragraph{Additional results.} Additional detailed results on Atari games are in Table~\ref{table:atarinodelay} and Table~\ref{table:ataridelayed}. In both tables, we show the performance of different algorithmic variants (first-order, second-order, V-trace) across all Atari games after training for $400M$ frames. In Table \ref{table:atarinodelay}, there is no artificial delay between actors and the learner, though there is still delay due to the computational setup across multiple machines. In Table \ref{table:ataridelayed}, there is an artificial delay between actors and the learner.

\paragraph{Details on the distributed architecture.} The general policy-based distributed agent follows the architecture design of IMPALA \citep{espeholt2018impala}, i.e., a central GPU learner and $N\triangleq512$ distributed CPU actors. The actors keep generating data by executing their local copies of the policy $\mu$, and sending data to the queue maintained by the learner. The parameters are periodically synchronized between the actors and the learner.

The architecture details are the same the ones of \citet{espeholt2018impala}. For completeness, we give some important details below; please refer to the original paper for the full description. For the delay experiments (Figure~\ref{fig:impala}), we  used two different model architectures: a shallow model based on work of \citet{mnih2016asynchronous} with an LSTM between the torso embedding and the output of policy/value function. The deep model refers to a deep network model with residual network \citep{he2016deep}. See Figure~3 of \citep{espeholt2018impala} for details, in particular the layer size and activation's functions.

The policy/value function networks are both trained with RMSProp optimizers \citep{tieleman2012lecture} with learning rate $\alpha \triangleq 5\cdot10^{-4}$ and no momentum. To encourage exploration, the policy loss is augmented by an entropy regularization term with coefficient $c_e \triangleq  0.01$ and a  baseline loss with coefficient $c_v\triangleq0.5$, i.e., the full loss is $L \triangleq L_{\text{policy}} + c_vL_\text{value} + c_e L_{\text{entropy}}$. These single hyper-parameters are selected according to Appendix~D of \citet{espeholt2018impala}.

Actors send partial trajectories of length $T\triangleq20$ to the learner.
For robustness of the training, rewards $r_t$ are clipped between $[-1,1]$. We adopt frame stacking and sticky actions as \citet{mnih2013playing}. The discount factor is $\gamma\triangleq0.99$ for calculating the baseline estimations.

\begin{table}[]
    \vskip 0.1in
    \begin{center}
    \footnotesize
    \begin{sc}
    \begin{tabular}{c|c|c|c|c|c}\toprule[1.5pt]
        \bf Levels & \bf Random & \bf Human & \bf V-trace & \bf First-order & \bf Second-order (TayPO-$2$) \\\midrule
alien & 227.75 & 7127.8 & 11358 & 5004 & \textbf{9634} \\ \hline
amidar & 5.77 & \textbf{1719.53} & 1442 & 1368 & 1350 \\ \hline
assault & 222.39 & 742 & 13759 & 9930 & \textbf{11505} \\ \hline
asterix & 210 & 8503.33 & 135730 & 152980 & \textbf{170490} \\ \hline
asteroids & 719.1 & \textbf{47388.67} & 29545 & 35385 & 44015 \\ \hline
atlantis & 12850 & 29028.13 & 711170 & \textbf{724230} & 700410 \\ \hline
bank\_heist & 14.2 & 753.13 & 1188 & 1166 & \textbf{1218} \\ \hline
battle\_zone & 2360 & \textbf{37187.5} & 13370 & 13828 & 13755 \\ \hline
beam\_rider & 363.88 & 16926.53 & \textbf{24031} & 18798 & 23735 \\ \hline
berzerk & 123.65 & \textbf{2630.42} & 1292 & 1383 & 1347 \\ \hline
bowling & 23.11 & \textbf{160.73} & 50 & 50 & 53 \\ \hline
boxing & 0.05 & 12.06 & 99 & 99 & 99 \\ \hline
breakout & 1.72 & 30.47 & 551 & 580 & \textbf{637} \\ \hline
centipede & 2090.87 & \textbf{12017.04} & 10166 & 8773 & 7747 \\ \hline
chopper\_command & 811 & 7387.8 & \textbf{19256} & 17129 & 17776 \\ \hline
crazy\_climber & 10780.5 & 35829.41 & \textbf{139190} & 132670 & 134310 \\ \hline
defender & 2874.5 & 18688.89 & 73020 & 72658 & 133090 \\ \hline
demon\_attack & 152.07 & 1971 & 119130 & 117860 & \textbf{133030} \\ \hline
double\_dunk & -18.55 & -16.4 & -7.6 & \textbf{-7.4} & -8.5 \\ \hline
enduro & 0 & \textbf{860.53} & 0 & 0 & 0 \\ \hline
fishing\_derby & -91.71 & -38.8 & \textbf{33} & 32 & 31.4 \\ \hline
freeway & 0.01 & \textbf{29.6} & 0 & 0 & 0 \\ \hline
frostbite & 65.2 & \textbf{4334.67} & 302 & 298 & 302 \\ \hline
gopher & 257.6 & 2412.5 & 23232 & 20805 & \textbf{26123} \\ \hline
gravitar & 173 & \textbf{3351.43} & 373 & 386 & 430 \\ \hline
hero & 1026.97 & 30826.38 & 32757 & 33277 & \textbf{36639} \\ \hline
ice\_hockey & -11.15 & 0.88 & 0.7 & 1.6 & \textbf{4.3} \\ \hline
jamesbond & 29 & 302.8 & \textbf{759} & 548 & 693 \\ \hline
kangaroo & 52 & \textbf{3035} & 1147 & 1339 & 1181 \\ \hline
krull & 1598.05 & 2665.53 & 9545 & 8408 & \textbf{9971} \\ \hline
kung\_fu\_master & 258.5 & 22736.25 & \textbf{44920} & 33004 & 41516 \\ \hline
montezuma\_revenge & 0 & \textbf{4753.33} & 0 & 0 & 0 \\ \hline
ms\_pacman & 307.3 & 6951.6 & 4018 & 4982 & \textbf{9702} \\ \hline
name\_this\_game & 2292.35 & 8049 & \textbf{18084} & 12345 & 13316 \\ \hline
phoenix & 761.4 & 7242.6 & \textbf{148840} & 91040 & 94131 \\ \hline
pitfall & -229.44 & \textbf{6463.69} & -5.9 & -4.2 & -4.5 \\ \hline
pong & -20.71 & 14.59 & 21 & 21 & 21 \\ \hline
private\_eye & 24.94 & \textbf{69571.27} & 100 & 94 & 99 \\ \hline
qbert & 163.88 & 13455 & 16044 & 20862 & \textbf{20891} \\ \hline
riverraid & 1338.5 & 17118 & \textbf{24116} & 22151 & 21253 \\ \hline
road\_runner & 11.5 & 7845 & 39513 & \textbf{43974} & 38177 \\ \hline
robotank & 2.16 & \textbf{11.94} & 7.2 & 7.1 & 7 \\ \hline
seaquest & 68.4 & \textbf{42054.71} & 1731 & 1735 & 1743 \\ \hline
skiing & -17098.09 & \textbf{-4336.93} & -10865 & -13303 & -10386 \\ \hline
solaris & 1236.3 & \textbf{12326.67} & 2375 & 2263 & 2486 \\ \hline
space\_invaders & 148.03 & 1668.67 & 13503 & \textbf{13544} & 13171 \\ \hline
star\_gunner & 664 & 10250 & \textbf{265480} & 190920 & 214580 \\ \hline
surround & -9.99 & \textbf{6.53} & 4.3 & 3.4 & 2.4 \\ \hline
tennis & -23.84 & -8.27 & 20.6 & \textbf{22} & 21.8 \\ \hline
time\_pilot & 3568 & 5229.1 & 28871 & \textbf{32813} & 32447 \\ \hline
tutankham & 11.43 & 167.59 & 243 & \textbf{278} & 277 \\ \hline
up\_n\_down & 533.4 & 11693.23 & \textbf{193520} & 163130 & 188190 \\ \hline
venture & 0 & \textbf{1187.5} & 0 & 0 & 0 \\ \hline
video\_pinball & 0 & 17667.9 & \textbf{359610} & 326060 & 315930 \\ \hline
wizard\_of\_wor & 563.5 & 4756.52 & 7302 & 5114 & \textbf{7646} \\ \hline
yars\_revenge & 3092.91 & 54576.93 & 81584 & 90581 & \textbf{93680} \\ \hline
zaxxon & 32.5 & 9173.3 & 21635 & 21149 & \textbf{25603} \\ 
         \bottomrule[1.46pt]
    \end{tabular} \par
    \end{sc}
    \end{center}
    \caption{Scores across 57 Atari levels for experiments on general policy-optimization with distributed architecture with \textbf{no artificial delays} between actors and learner. We compare several alternatives for off-policy correction: V-trace, first-order and second-order. We also provide scores for random policy and human players as reference. All scores are obtained by training for 400M frames. Best results per game are highlighted in bold font.}
    \label{table:atarinodelay}
\end{table}

\begin{table}[]
    \vskip 0.1in
    \begin{center}
    \footnotesize
    \begin{sc}
    \begin{tabular}{c|c|c|c|c|c}\toprule[1.5pt]
        \bf Levels & \bf Random & \bf Human & \bf V-trace & \bf First-order & \bf Second-order (TayPO-$2$) \\\midrule
alien & 227.75 & 7127.8 & 464 & 1820 & \textbf{3257} \\ \hline
amidar & 5.77 & \textbf{1719.53} & 81 & 428 & 541 \\ \hline
assault & 222.39 & 742 & 1764 & 4868 & \textbf{6490} \\ \hline
asterix & 210 & 8503.33 & 2151 & \textbf{165170} & 161800 \\ \hline
asteroids & 719.1 & \textbf{47388.67} & 2256 & 1329 & 3886 \\ \hline
atlantis & 12850 & 29028.13 & 311111 & 543210 & \textbf{621920} \\ \hline
bank\_heist & 14.2 & \textbf{753.13} & 71 & 483 & 524 \\ \hline
battle\_zone & 2360 & \textbf{37187.5} & 9021 & 10481 & 13820 \\ \hline
beam\_rider & 363.88 & 16926.53 & 7391 & 16769 & \textbf{19030} \\ \hline
berzerk & 123.65 & \textbf{2630.42} & 631 & 757 & 826 \\ \hline
bowling & 23.11 & \textbf{160.73} & 40 & 36 & 50 \\ \hline
boxing & 0.05 & 12.06 & 51 & 93 & \textbf{95} \\ \hline
breakout & 1.72 & 30.47 & 71 & 298 & \textbf{387} \\ \hline
centipede & 2090.87 & \textbf{12017.04} & 8847 & 6545 & 6924 \\ \hline
chopper\_command & 811 & 7387.8 & 2340 & 4837 & \textbf{8064} \\ \hline
crazy\_climber & 10780.5 & \textbf{35829.41} & 23745 & 63982 & 117830 \\ \hline
defender & 2874.5 & 18688.89 & 20594 & 18088 & \textbf{34684} \\ \hline
demon\_attack & 152.07 & 1971 & 36491 & 40324 & \textbf{63758} \\ \hline
double\_dunk & -18.55 & -16.4 & -11.7 & -9.9 & \textbf{-7.2} \\ \hline
enduro & 0 & \textbf{860.53} & 0 & 0 & 0 \\ \hline
fishing\_derby & -91.71 & -38.8 & -6.6 & 15.4 & \textbf{15.7} \\ \hline
freeway & 0.01 & \textbf{29.6} & 0 & 0 & 0.01 \\ \hline
frostbite & 65.2 & \textbf{4334.67} & 230 & 257 & 267 \\ \hline
gopher & 257.6 & 2412.5 & 1551 & 2213 & \textbf{5376} \\ \hline
gravitar & 173 & \textbf{3351.43} & 263 & 300 & 351 \\ \hline
hero & 1026.97 & \textbf{30826.38} & 2012 & 3452 & 12027 \\ \hline
ice\_hockey & -11.15 & 0.88 & -1.5 & -0.9 & \textbf{1.01} \\ \hline
jamesbond & 29 & 302.8 & 307 & \textbf{406} & 389 \\ \hline
kangaroo & 52 & \textbf{3035} & 416 & 342 & 805 \\ \hline
krull & 1598.05 & 2665.53 & 5737 & 5416 & \textbf{9101} \\ \hline
kung\_fu\_master & 258.5 & 22736.25 & 12991 & 12968 & \textbf{23741} \\ \hline
montezuma\_revenge & 0 & \textbf{4753.33} & 0 & 0 & 0 \\ \hline
ms\_pacman & 307.3 & \textbf{6951.6} & 960 & 2542 & 2763 \\ \hline
name\_this\_game & 2292.35 & 8049 & 13315 & 15510 & 15510 \\ \hline
phoenix & 761.4 & 7242.6 & 6538 & 16566 & \textbf{32146} \\ \hline
pitfall & -229.44 & \textbf{6463.69} & -4.5 & -4.5 & -3.2 \\ \hline
pong & -20.71 & 14.59 & -14 & 13 & \textbf{18.1} \\ \hline
private\_eye & 24.94 & \textbf{69571.27} & 88 & 80 & 185 \\ \hline
qbert & 163.88 & \textbf{13455} & 1155 & 8856 & 10578 \\ \hline
riverraid & 1338.5 & \textbf{17118} & 4607 & 2632 & 5064 \\ \hline
road\_runner & 11.5 & 7845 & 6404 & 16792 & \textbf{36857} \\ \hline
robotank & 2.16 & \textbf{11.94} & 6.2 & 5.5 & 8.07 \\ \hline
seaquest & 68.4 & \textbf{42054.71} & 1884 & 1881 & 2283 \\ \hline
skiing & -17098.09 & \textbf{-4336.93} & -27463 & -11778 & -22189 \\ \hline
solaris & 1236.3 & \textbf{12326.67} & 2435 & 2269 & 2320 \\ \hline
space\_invaders & 148.03 & 1668.67 & 1029 & 2955 & \textbf{4399} \\ \hline
star\_gunner & 664 & 10250 & 25622 & 27001 & \textbf{51257} \\ \hline
surround & -9.99 & \textbf{6.53} & -8.4 & -2.5 & -0.74 \\ \hline
tennis & -23.84 & -8.27 & -20 & -8.84 & \textbf{4.89} \\ \hline
time\_pilot & 3568 & 5229.1 & 8963 & \textbf{18295} & 17884 \\ \hline
tutankham & 11.43 & 167.59 & 97 & 161 & \textbf{172} \\ \hline
up\_n\_down & 533.4 & 11693.23 & 18726 & 18693 & \textbf{49468} \\ \hline
venture & 0 & \textbf{1187.5} & 0 & 0 & 0 \\ \hline
video\_pinball & 0 & 17667.9 & \textbf{28962} & 210960 & 191240 \\ \hline
wizard\_of\_wor & 563.5 & 4756.52 & 4142 & 5234 & \textbf{5349} \\ \hline
yars\_revenge & 3092.91 & \textbf{54576.93} & 3375 & 26302 & 29403 \\ \hline
zaxxon & 32.5 & 9173.3 & 6251 & 9040 & \textbf{9359} \\ 
         \bottomrule[1.46pt]
    \end{tabular} \par
    \end{sc}
    \end{center}
    \caption{Scores across 57 Atari levels for experiments on general policy-optimization with distributed architecture with \textbf{severe delays} between actors and learner. We compare several alternatives for off-policy correction: V-trace, first-order and second-order. We also provide scores for random policy and human players as reference. All scores are obtained by training for 400M frames. The performance across all algorithms generally degrade significantly compared to Table \ref{table:atarinodelay}, the second-order degrades more gracefully than other baselines.  Best results per game are highlighted in bold.}
    \label{table:ataridelayed}
\end{table}

\subsection{Distributed value-based learning}
\label{appendix:experiment-valuelearning}

\paragraph{Hyper-parameters for Taylor expansions.} The algorithmic details (e.g., the expression for recursive updates) are specified  in Appendix~\ref{appendix:secondorderr2d2}. Given a partial trajectory, the zero-order variant calculates the targets recursively along the entire trajectory. For first-order and second-order variants, we find that calculating the targets recursively along the entire trajectory tends to destabilize the updates. We suspect that this is because the function approximation error accumulates along the recursive computation, leading to very poor estimates at the beginning of the partial trajectory. Note that this is very different from update rules such as Retrace \citep{munos2016safe}, where the trace coefficient $c_t \triangleq \lambda\min\{\bar{c},\pi_t/\mu_t\}$ tends to be zero frequently because $\pi_t$ is a greedy policy, traces are cut automatically and function approximation errors do not accumulate as much along the trajectory. For Taylor expansion variants with order $K\geq 2$, the trace coefficient is effectively $c_t \triangleq 1$ and the trace is not cut at all. To remedy such an issue, we compute corrected n-step updates with $n\triangleq3$. This ensures that the errors do not propagate up to $n$ steps and stabilize the learning process. 

Importantly, we note that the accumulation of errors along trajectories might also happen for policy-based algorithms. However, we speculate that policy-based agents are more robust to such errors because it is the relative values which influence the direction of policy updates. See Appendix \ref{appendix:experiment-offpolicypg} for details on policy-based algorithms.

In the experiments, we found $\eta=0.2$ to work the best. This best hyper-parameter was selected across $\eta \in \{0, 0.2, 0.5, 0.8, 1.0\}$ where $\eta=0$ corresponds to the first-order. Note that this best hyper-parameter differs from those of previous experiments with policy-based agents. This means that carrying out the full second-order expansion does not outperform the first-order; the best outcome is obtained in the middle.


\paragraph{Additional results.} We provide additional results on Atari games in Figure \ref{fig:r2d2all}, where in order to present a more complete picture of the training properties of different algorithmic variants, we provide mean/median/super-human ratio of the human-normalized scores. At each point of the training (e.g., fixing a number of training frames), we have access to the full set of human-normalized scores $z_i,1\leq i\leq 57$. Then, the three statistics are computed as usual across these scores. The super-human ratio is computed as 
the proportion of games such that $z_i > 1$, i.e., such that the learning algorithm reaches super-human performance.

Overall, we see that the second-order expansion provides benefits in terms of the mean performance. In median performance, first-order and second-order are very similar, both providing a slight advantage over Retrace. Across these two statistics, the zero-order achieves the worst results, since the performance plateaus at a low level. However, the super-human ratio statistics implies that the zero-order variant can achieve super-human performance on almost all games as quickly as other more complex variants. 

\paragraph{Details on the distributed architecture.}
We follow the architecture designs of R2D2 \citep{kapturowski2018recurrent}. We recap the important details for completeness. For a complete description, please refer to the original paper. 

The agent contains a single GPU learner and $256$ CPU actors. The policy/value network applies the same architecture  as \citep{mnih2016asynchronous}, with a 3-layer convnet followed by an LSTM with 512 hidden units, whose output is fed into a dueling head (with hidden layer size of 512, \citealt{wang2015dueling}). Importantly, to leverage the recurrent architecture, each time step consists of the current observation frame, the reward and one-hot action embedding from the previous time step. Note that here we do no stack frames as practiced in e.g., IMPALA \citep{espeholt2018impala}.

The actor sends partial trajectories of length $T\triangleq120$ to the replay buffer. Here, the first $T_1\triangleq40$ steps are used for burn-in while the rest $T_2\triangleq80$ steps are used for loss computations. The replay buffer can hold $4\cdot10^6$ time steps and replays according to a priority exponent of $0.9$ and IS exponent of $0.6$ \cite{horgan2018distributed}. The actor synchronizes parameters from the learner every 400 environment time steps. 

To calculate Bellman updates, we take a very high discount factor $\gamma=0.997$. To stabilize the training, a target network is applied to compute the target values. The target network is updated every $2500$ gradient updates of the main network. We also apply a hyperbolic transform in calculating the Bellman target \citep{pohlen2018observe}.

All networks are optimized by an Adam optimizer \citep{kingma2014adam} with learning rate $\alpha \triangleq 10^{-4}$.

\begin{figure}[t]
\includegraphics[width=18cm]{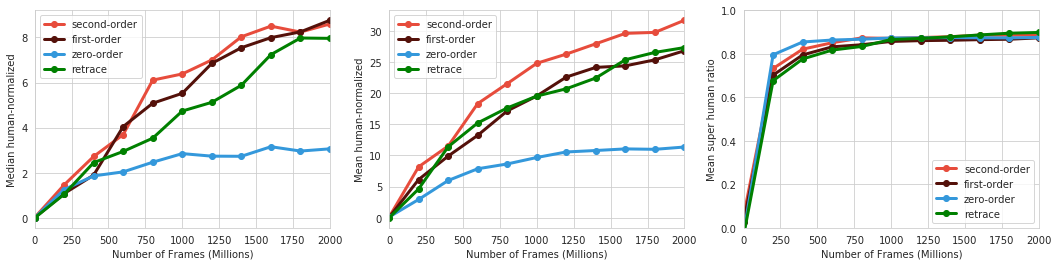}
\centering
\caption{Value-based learning with distributed architecture (R2D2). The x-axis is number of frames (millions) and y-axis shows the mean/median/super-human ratio of human-normalized scores averaged across 57 Atari levels over the training of 2000M frames. Each curve averages across 2 random seeds. The second-order correction performs marginally better than first-order correction and retrace, and significantly better than zero-order. The super-human ratio is computed as the proportion of games with normalized scores $z_i > 1$.}
\label{fig:r2d2all}
\end{figure}

\subsection{Ablation study}
\label{appendix:experiment-ablation}

In this part we study the impact of the hyper-parameter $\eta$ on the performance of algorithms derived from second-order expansion.  In particular, we study the effect of $\eta$ in the near on-policy optimization as in the context of Section~\ref{onpolicy}. In Figure~\ref{fig:eta}, x-axis shows the training frames (400M in total) and y-axis shows the mean human-normalized scores across Atari games. We select $\eta \in \{0.5, 1.0, 1.5\}$ and compare their training curves. We find that when $\eta$ is selected within this range, the training performance does not change much, which hints on some robustness with respect to~$\eta$. Inevitably, when $\eta$ takes extreme values the performance degrades. When $\eta=0$ the algorithm reduces to the first-order case and the performance gets marginally worse as discussed in the main text.

\paragraph{Value-based learning.} The effect of $\eta$ on value-based learning is different from the case of policy-based learning. Since the second-order expansion partially corrects for the value function estimates, its effect becomes more subtle for value-based algorithms such as R2D2. See discussions in Appendix \ref{appendix:secondorderr2d2}.

\begin{figure}[t]
\includegraphics[width=8cm]{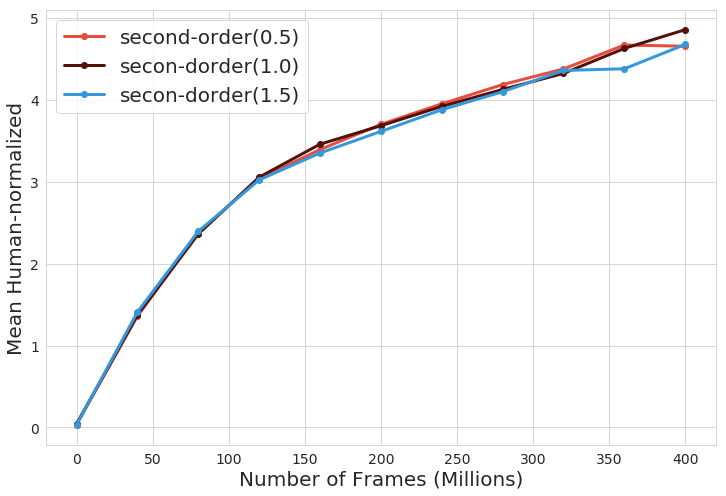}
\centering
\caption{Ablation study on the effect of varying $\eta$. The x-axis shows the training frames (a total of 400M frames) and y-axis shows the mean human-normalized scores averaged across all Atari games. We select $\eta \in \{0.5,1.0,1.5\}$. In the legend, numbers in the brackets indicate the value of $\eta$.}
\label{fig:eta}
\end{figure}

\end{document}